\documentclass[12pt]{article}

\usepackage{setspace}

\usepackage{geometry}                
\usepackage{graphicx}
\usepackage{amssymb}
\usepackage{amsmath}
\usepackage{mathrsfs} 
\usepackage{color}
\usepackage{hyperref}
\hypersetup{
	colorlinks=true, 
	linktoc=all,     
	linkcolor=blue,  
	citecolor=blue,
}

\usepackage{appendix}
\usepackage{tikz}
\usepackage[numbers]{natbib}

\usetikzlibrary{positioning}
\usetikzlibrary{arrows}
\newtheorem{remark}{Remark}

\newtheorem{theorem}{Theorem}[section]
\newtheorem{corollary}{Corollary}[section]
\newtheorem{lemma}{Lemma}[section]

\newenvironment{proof}{{\noindent\it Proof}\quad}{\hfill $\square$\par} 

\numberwithin{figure}{section}
\numberwithin{equation}{section}

\usepackage{algorithm}
\usepackage{algorithmic}
\usepackage{bm}

\usepackage{stmaryrd}

\begin{document}

\title{Approximation Properties of Deep ReLU CNNs}
\author{Juncai He\footnotemark[1] \quad Lin Li\footnotemark[2]\quad  Jinchao Xu\footnotemark[3]}
\date{} 

\maketitle
\renewcommand{\thefootnote}{\fnsymbol{footnote}} 
\footnotetext[1]{Department of Mathematics, The University of Texas at Austin, Austin, TX 78712, USA (jhe@utexas.edu).} 
\footnotetext[2]{Beijing International Center for Mathematical Research, Peking University,
	Beijing 100871, China (lilin1993@pku.edu.cn). } 
\footnotetext[3]{Department of Mathematics, The Pennsylvania State University, University Park, PA 16802, USA (xu@math.psu.edu).} 

\maketitle
\begin{abstract}
This paper focuses on establishing $L^2$ approximation properties for deep ReLU convolutional neural networks (CNNs) in two-dimensional space.  
The analysis is based on a decomposition theorem for convolutional kernels with a large spatial size and multi-channels.
Given the decomposition result, the property of the ReLU activation function, and a specific structure for channels, a universal approximation theorem of deep ReLU CNNs with classic structure is obtained by showing its connection with one-hidden-layer ReLU neural networks (NNs).  
Furthermore, approximation properties are obtained for one version of neural networks with ResNet, pre-act ResNet, and MgNet architecture based on connections between these networks.
\end{abstract}


\section{Introduction}\label{sec:intro}
The purpose of this paper is to study the approximation properties of deep convolutional neural networks, including 
classic CNNs~\cite{lecun1998gradient,krizhevsky2012imagenet}, ResNet~\cite{he2016deep}, pre-act
ResNet~\cite{he2016identity}, and MgNet~\cite{he2019mgnet}.
CNN is a very efficient deep learning model~\cite{lecun2015deep,goodfellow2016deep}, which has been widely used in image processing, computer vision, reinforcement learning, and
also scientific computing~\cite{guo2016convolutional,karniadakis2021physics}. 
However, there is still very little mathematical analysis of CNNs and, therefore, limited understanding of them, especially for the 
approximation property of CNNs, which plays a functional role in their interpretation and development~\cite{kohler2020statistical,lin2021universal}.

In the last three decades, researchers have produced a large number of studies on the approximation and representation properties of fully connected neural networks with a single hidden 
layer~\cite{hornik1989multilayer,cybenko1989approximation,barron1993universal,leshno1993multilayer,bach2017breaking,klusowski2018approximation,siegel2020approximation,xu2020finite,siegel2022high} and deep neural networks (DNNs) with more than one hidden layer~\cite{montufar2014number,telgarsky2016benefits,poggio2017and,yarotsky2017error,lu2017expressive,arora2018understanding,shen2019nonlinear,he2020relu,guhring2020error,opschoor2020deep,he2021relu}. 
To our knowledge, however, the literature included very few studies on the approximation property of CNNs~\cite{bao2014approximation,zhou2018deep,oono2019approximation,zhou2020universality,petersen2020equivalence,kumagai2020universal}. 
In \cite{bao2014approximation}, the authors consider a type of ReLU CNN with one-dimensional (1D) input that is constituted by a sequence of convolution layers and a fully connected output layer. By showing that the identity operator can be realized by an underlying sequence of convolutional layers, they obtain the approximation property of the CNN directly from the fully connected layer.
In their analysis, the underlying convolutional layers do not contribute anything to the approximation power of the overall CNN.
Approximation properties with more standard CNN architecture have been studied in \cite{zhou2018deep,zhou2020universality} in relation to the kernel decomposition for 1D convolutional operation with periodic padding. This type of result, however, cannot be extended to CNNs with two-dimensional (2D) inputs, because essentially it is the polynomial decomposition theory~\cite{daubechies1992ten} for 1D.
In \cite{zhou2018deep,zhou2020universality}, the authors study a standard 1D ReLU CNN architecture consisting of a sequence of convolution layers and a linear layer and obtain approximation properties by showing that any fully connected layer can be decomposed as a sequence of convolution layers with the ReLU activation function. 
In \cite{petersen2020equivalence},  the authors extend the analysis in \cite{zhou2018deep,zhou2020universality} to 2D  ReLU CNNs with periodic padding for a very special function class. This class is in the form of $f(X)=[F(X)]_{1,1}+b$ in which $F: \mathbb{R}^{d\times d} \mapsto \mathbb{R}^{d\times d}$ satisfies the following translation invariant property
\begin{equation}
	F( S_{st}(X)) = S_{st}(F(X)), \quad\forall X\in \mathbb R^{d \times d},
\end{equation}
where $S_{st}: \mathbb{R}^{d\times d} \mapsto \mathbb{R}^{d\times d}$ is considered the translation operator defined as $\left[S_{st}(X)\right]_{i,j} = [X]_{i-s,j-t}$ for $1\le s,t \le d$ with periodic padding.
Here, $[Y]$ means taking the element of the tensor $Y$.
A generalized study of this function class and its application in approximation properties of CNNs can be found in~\cite{kumagai2020universal}.
In \cite{oono2019approximation}, the authors study the approximation properties of ResNet-type CNNs on 1D for the special function class that can be approximated by sparse NNs. 

First, we show a pure algebraic decomposition theorem, which plays a critical role in establishing the approximation theorem of deep ReLU CNNs, for 2D convolutional kernels with multi-channel and constant or periodic padding. 
The core idea in establishing such a decomposition result is to introduce channels, whereas the decomposition theorem in \cite{bao2014approximation,zhou2018deep,zhou2020universality} incorporates only one channel.
By applying a similar argument in~\cite{zhou2020universality}, we then establish a connection between one-hidden-layer ReLU NNs and deep ReLU CNNs without pooling layers. 
According to this connection, we prove the approximation theorem of classic deep ReLU CNNs,
which shows that this kind of CNN can provide the same asymptotic approximation rate as {one-hidden-layer} ReLU {NNs}.
Moreover, we obtain approximation results for ResNet and pre-act ResNet CNNs by studying connections between classic deep ReLU CNNs and CNNs with ResNet or pre-act ResNet architecture.
Finally, we establish the approximation property of one version of MgNet~\cite{he2019mgnet} based on its connection with pre-act ResNet.

The paper is organized as follows. In Section~\ref{sec:decomp_kernel}, we introduce the {2D} convolutional operation with multi-channel and {paddings} and then prove
the decomposition theorem for large convolutional kernels. In Section~\ref{sec:approx_cnn}, we show
the approximation results for functions represented by classic CNNs without pooling operators. 
In Section~\ref{sec:concl}, we provide concluding remarks.

\section{Decomposition theorem of large convolutional kernels in CNNs}\label{sec:decomp_kernel}
In this section, we introduce the decomposition theorem for standard two-dimensional convolutional 
kernels with large spatial size.

First, let us follow the setup for the dimensions of the tensors in PyTorch~\cite{paszke2019pytorch} to denote {the} data with $c$ channels as 
\begin{equation}
	X\in \mathbb{R}^{c\times d\times d}
\end{equation}
with elements $[X]_{p,m,n}$ for $p=1:c$ and $m,n=1:d$.
For the convolutional kernel with input channel $c$, output channel $C$, and spatial size $(2k+1)\times(2k+1)$, we have
\begin{equation}
	K \in \mathbb{R}^{c\times C\times (2k+1)\times(2k+1)}
\end{equation}
with elements $[K]_{p,q,s,t}$ for $p=1:c$, $q=1:C$, and $s,t=-k:k$.
Then the standard multi-channel convolution operation in typical 2D CNNs~\cite{goodfellow2016deep} with constant or periodic padding is defined as $K \ast X \in  \mathbb{R}^{C\times d \times d}$ where
\begin{equation}\label{eq:defineConv}
	[K \ast X]_{q,m,n} = \sum_{p = 1}^c \sum_{s,t = -k}^{k} [K]_{p,q,s,t} [X]_{p,m + s, n+t}
\end{equation}
for $q = 1:C$ and $m,n  = 1: d$.
If the index $m + s$ or $ n+t$ exceeds the range $1:d$ in \eqref{eq:defineConv}, we denote the constant padding and the periodic padding as follows:
\begin{description}
	\item[Constant padding:] 
	\begin{equation}\label{eq:paddingC}
		[X]_{p,m+s,n+t}  = a,
	\end{equation}
	where $a \in \mathbb{R}$ is an arbitrary constant and $m+s \notin 1:d$ or  $n+t \notin 1:d$;
	\item[Periodic padding:] 
	\begin{equation}\label{eq:paddingP}
		[X]_{p,m+s,n+t}  = [X]_{p,k, l},
	\end{equation}
	where $1\le k,l \le d$, $k \equiv m+s ~({\rm mod} ~{d})$, and $l \equiv n+t ~({\rm mod} ~{d})$.
\end{description}
The convolution with constant or periodic padding defined in~\eqref{eq:defineConv}, {referred to as} convolution with stride one~\cite{goodfellow2016deep} with padding, is the most commonly used convolutional operation in practical CNNs~\cite{he2016deep,he2016identity,huang2017densely,tan2019efficientnet}. 
An important feature of this convolution is that the spatial {dimensions} of its inputs do not change in the presence of paddings.
\begin{remark}
	For simplicity, we assume the index of the convolution kernel $K\in \mathbb{R}^{(2k+1) \times (2k+1)}$ starts from $-k$ and ends at $k$, whereas the index of the data or tensor after convolution starts from $1$. 
	In addition, we stress that the convolution operation defined above follows 
	neither the commutative law nor the associative law. Thus, we mean
	\begin{equation}	
		K_2 \ast K_1 \ast X := K_2 \ast \left( K_1 \ast X\right)
	\end{equation}
	by default.
\end{remark}

Our study begins with the observation that a $5\times5$ kernel can be represented 
by the combination of two composed $3\times3$ kernels.
\begin{lemma}\label{lemm:decomp_K_5}
	Let $K \in \mathbb{R}^{5\times 5}$ and $d>2$, then there exist $P_{i,j}, S_{i,j} \in \mathbb{R}^{3\times 3}$ for $i,j=-1,0,1$ such 
	that
	\begin{equation}\label{eq:decomp5_3}
		K \ast X = \sum_{i,j=-1,0,1} P_{i,j} \ast S_{i,j} \ast X, \quad \forall X \in \mathbb{R}^{d\times d},
	\end{equation}
	where $\ast$ means {the} standard convolution with {one channel} and constant or periodic padding as in~\eqref{eq:defineConv}.
\end{lemma}
\begin{proof}
	Here, we {present} a constructive proof by taking $S_{i,j}$ as
	\begin{equation}\label{eq:S_i}
		\left[S_{i,j}\right]_{s,t} = \begin{cases}
			1 \quad &s=i \text{ and } t=j, \\
			0 \quad &\text{others},
		\end{cases}
	\end{equation}
	i.e.,
	\begin{scriptsize}
		\begin{equation}\label{eq:S_i_detial}
			\begin{aligned}
				&S_{-1,-1} = 
				\begin{pmatrix}
					1 & 0 & 0\\
					0 & 0 & 0\\
					0 & 0 & 0
				\end{pmatrix},	
				&&S_{-1,0} = 
				\begin{pmatrix}
					0 & 1 & 0\\
					0 & 0 & 0\\
					0 & 0 & 0
				\end{pmatrix}, 
				&&S_{-1,1} = 
				\begin{pmatrix}
					0 & 0 & 1\\
					0 & 0 & 0\\
					0 & 0 & 0
				\end{pmatrix},
				\\
				&S_{0,-1} = 
				\begin{pmatrix}
					0 & 0 & 0\\
					1 & 0 & 0\\
					0 & 0 & 0
				\end{pmatrix},
				&&S_{0,0} = 
				\begin{pmatrix}
					0 & 0 & 0\\
					0 & 1 & 0\\
					0 & 0 & 0
				\end{pmatrix},
				&&S_{0,1} = 
				\begin{pmatrix}
					0 & 0 & 0\\
					0 & 0 & 1\\
					0 & 0 & 0
				\end{pmatrix}, 
				\\
				&S_{1,-1} = 
				\begin{pmatrix}
					0& 0 & 0\\
					0& 0 & 0\\
					1 & 0 & 0
				\end{pmatrix},
				&&S_{1,0} = 
				\begin{pmatrix}
					0 & 0 & 0\\
					0 & 0 & 0\\
					0 & 1 & 0
				\end{pmatrix},
				&&S_{1,1} = 
				\begin{pmatrix}
					0 & 0 & 0 \\
					0 & 0 & 0 \\
					0 & 0 & 1
				\end{pmatrix},
			\end{aligned}
		\end{equation}
	\end{scriptsize}
	and $P_{i,j}$ as
	\begin{scriptsize}
		\begin{equation}\label{eq:P_i_5}
			\begin{aligned}
				&P_{-1,-1} = 
				\begin{pmatrix}
					K_{-2,-2} & 0 & 0\\
					0 & 0 & 0\\
					0 & 0 & 0
				\end{pmatrix},	
				&&P_{-1,0} = 
				\begin{pmatrix}
					K_{-2,-1} & K_{-2,0} & K_{-2,1}\\
					0 & 0 & 0\\
					0 & 0 & 0
				\end{pmatrix}, 
				&&P_{-1,1} = 
				\begin{pmatrix}
					0 & 0 & K_{-2,2}\\
					0 & 0 & 0\\
					0 & 0 & 0
				\end{pmatrix},
				\\
				&P_{0,-1} = 
				\begin{pmatrix}
					K_{-1,-2} & 0 & 0\\
					K_{0,-2} & 0 & 0\\
					K_{1,-2} & 0 & 0
				\end{pmatrix},
				&&P_{0,0} = 
				\begin{pmatrix}
					K_{-1,-1} & K_{-1,0} & K_{-1,1}\\
					K_{0,-1} & K_{0,0} & K_{0,1}\\
					K_{1,-1} & K_{1,0} & K_{1,1}
				\end{pmatrix},
				&&P_{0,1} = 
				\begin{pmatrix}
					0 & 0 & K_{-1,2} \\
					0 & 0 & K_{0,2} \\
					0 & 0 & K_{1,2}
				\end{pmatrix}, 
				\\
				&P_{1,-1} = 
				\begin{pmatrix}
					0& 0 & 0\\
					0& 0 & 0\\
					K_{2,-2} & 0 & 0
				\end{pmatrix},
				&&P_{1,0} = 
				\begin{pmatrix}
					0 & 0 & 0\\
					0 & 0 & 0\\
					K_{2,-1} & K_{2,0} & K_{2,1}
				\end{pmatrix},
				&&P_{1,1} = 
				\begin{pmatrix}
					0 & 0 & 0 \\
					0 & 0 & 0 \\
					0 & 0 & K_{2,2}
				\end{pmatrix}.
			\end{aligned}
		\end{equation}
	\end{scriptsize}
	
\end{proof}
\begin{remark}\label{rem:PS}
	$S_{i,j}$ and $P_{i,j}$ can be collected separately to form two multi-channel convolution kernels. More precisely, we have
	\begin{equation}\label{key}
		S = (S_{-1,-1}, S_{-1,0}, \cdots, S_{1,1}) \in \mathbb{R}^{1 \times 9 \times 3 \times 3}
	\end{equation}
	and 
	\begin{equation}\label{key}
		P = (P_{-1,-1}, P_{-1,0}, \cdots, P_{1,1})^T \in \mathbb{R}^{9 \times 1 \times 3 \times 3}.
	\end{equation}
	That is, the convolution operation defined in \eqref{eq:decomp5_3} can be written as
	\begin{equation}\label{key}
		K \ast X = P \ast S \ast X.
	\end{equation} 
\end{remark}
Then, the most critical step is to extend Lemma~\ref{lemm:decomp_K_5} {to} a convolutional kernel $K \in \mathbb{R}^{(2k+1)\times (2k+1)}$ with large spatial size; i.e., $k$ is large. 
Thus, we introduce the next decomposition for any $K \in \mathbb{R}^{(2k+1)\times (2k+1)}$ as
\begin{equation}\label{eq:decom_tildeK}
	K = \sum_{i,j=-1,0,1} \widetilde{K}_{i,j},
\end{equation}
where
\begin{footnotesize}
	\begin{equation}\label{eq:decom_tildeK_1}
		\begin{split}
			\widetilde{K}_{-1,-1} &= \begin{pmatrix}
				K_{-k,-k} &0& \cdots \\
				0 & \ddots & \vdots\\
				\vdots & \cdots & 0
			\end{pmatrix} 
			=\begin{pmatrix}
				P_{-1,-1} & \begin{matrix}
					0 & 0 \\ \vdots & \vdots
				\end{matrix} \\
				\begin{matrix}
					0 & \cdots \\
					0 & \cdots
				\end{matrix} & \begin{matrix}
					0&0\\0&0
				\end{matrix}
			\end{pmatrix}
			\in  \mathbb{R}^{(2k+1)\times (2k+1)}, \\
			\widetilde K_{-1,0} &= 
			\begin{pmatrix}
				0 &K_{-k,-k+1} & \cdots & K_{-k,k-1} & 0 \\
				0& \ddots &0 & \ddots &0\\
				\vdots &  \cdots & \vdots & \cdots &\vdots 
			\end{pmatrix} 
			= \begin{pmatrix}
				0 & P_{-1,0} &  0 \\
				\vdots & & \vdots \\
				0 & \cdots & 0
			\end{pmatrix}
			\in  \mathbb{R}^{(2k+1)\times (2k+1)}, \\
			&\vdots \\
			\widetilde{K}_{1,1} & = 	
			\begin{pmatrix}
				0 & \cdots & \vdots \\
				\vdots & \ddots & 0 \\
				\cdots & 0 & K_{k,k}
			\end{pmatrix} =
			\begin{pmatrix}
				\begin{matrix}
					0&0\\0&0
				\end{matrix} & \begin{matrix}
					\cdots & 0 \\
					\cdots &0
				\end{matrix} \\
				\begin{matrix}
					\vdots & \vdots \\ 0 & 0
				\end{matrix}
				& P_{1,1}
			\end{pmatrix}
			\in  \mathbb{R}^{(2k+1)\times (2k+1)},
		\end{split}
	\end{equation}
\end{footnotesize}
and  $P_{i,j} \in \mathbb{R}^{(2k-1)\times(2k-1)}$ with
\begin{tiny}
	\begin{equation}\label{eq:P_i_2k+1}
		\begin{aligned}
			&P_{-1,-1} = 
			\begin{pmatrix}
				K_{-k,-k} &0& \cdots \\
				0 & \ddots & \vdots\\
				\vdots & \cdots & 0
			\end{pmatrix},	
			&&P_{-1,0} = 
			\begin{pmatrix}
				K_{-k,-k+1} & \cdots & K_{-k,k-1}\\
				0& \ddots &0\\
				\vdots &  \cdots & \vdots 
			\end{pmatrix}, 
			&&P_{-1,1} = 
			\begin{pmatrix}
				\cdots & 0 & K_{-k,k}\\
				\vdots & \ddots & 0\\
				0 & \cdots & \vdots
			\end{pmatrix},
			\\
			&P_{0,-1} = 
			\begin{pmatrix}
				K_{-k+1, -k} & 0 & \cdots\\
				\vdots & \ddots & \vdots\\
				K_{k-1, -k} & 0 & \cdots
			\end{pmatrix},
			&&P_{0,0} = 
			\begin{pmatrix}
				K_{-k+1,-k+1} & \cdots & K_{-k+1,k-1}\\
				\vdots & \ddots & \vdots\\
				K_{k-1,-k+1} & \cdots & K_{k-1,k-1}
			\end{pmatrix},
			&&P_{0,1} = 
			\begin{pmatrix}
				\cdots & 0 & K_{-k+1,k} \\
				\vdots & \ddots & \vdots \\
				\cdots & 0 & K_{k-1,k}
			\end{pmatrix}, 
			\\
			&P_{1,-1} = 
			\begin{pmatrix}
				\vdots & \cdots & 0\\
				0& \ddots & \vdots\\
				K_{k,-k} & 0 & \cdots
			\end{pmatrix},
			&&P_{1,0} = 
			\begin{pmatrix}
				\vdots & \cdots & \vdots\\
				0 & \ddots & 0\\
				K_{k,-k+1} & \cdots & K_{k,k-1}
			\end{pmatrix},
			&&P_{1,1} = 
			\begin{pmatrix}
				0 & \cdots & \vdots \\
				\vdots & \ddots & 0 \\
				\cdots & 0 & K_{k,k}
			\end{pmatrix}.
		\end{aligned}
	\end{equation}
\end{tiny}
A more intuitive description {of} the previous decomposition is
\begin{equation}
	K = 	\begin{pmatrix}
		\boxed{K_{-k,-k}} & \boxed{\begin{matrix}
				K_{-k,-k+1} & \cdots & K_{-k,k-1} 
		\end{matrix}} & \boxed{K_{-k,k}} \\ 
		~&~&~\\
		\boxed{\begin{matrix}
				K_{-k+1,-k} \\ \vdots \\K_{k-1,-k}
		\end{matrix}} & 
		\boxed{\begin{matrix}
				K_{-k+1,-k+1} & \cdots & K_{-k+1,k-1}\\
				\vdots & \ddots & \vdots\\
				K_{k-1,-k+1} & \cdots & K_{k-1,k-1}
		\end{matrix}} & \boxed{\begin{matrix}
				K_{-k+1,k} \\ \vdots \\K_{k-1,k}
		\end{matrix}} \\
		~&~&~\\
		\boxed{K_{k,-k}} & \boxed{\begin{matrix}
				K_{k,-k+1}&\cdots &K_{k,k-1}
		\end{matrix}} & \boxed{K_{k,k}}
	\end{pmatrix}.
\end{equation}
Thus, we {can regard} $P_{i,j}$ in \eqref{eq:P_i_2k+1} as the generalization of $P_{i,j}$ in
\eqref{eq:P_i_5}.
Now, we {present} the main theorem for decomposing any large convolutional kernels $K \in \mathbb{R}^{(2k+1)\times (2k+1)}$.
\begin{theorem}\label{thm:decomp2k+1}
	Let $K \in \mathbb{R}^{(2k+1)\times (2k+1)}$ and $d > k$.
	Then we can take $P_{i,j} \in \mathbb{R}^{(2k-1)\times(2k-1)}$ as in \eqref{eq:P_i_2k+1} and $S_{i,j} \in \mathbb{R}^{3\times 3}$ as in \eqref{eq:S_i} for $i,j=-1,0,1$ such that
	\begin{equation}\label{eq:decomp_k}
		K \ast X = \sum_{i,j=-1,0,1} P_{i,j}\ast S_{i,j} \ast X, \quad \forall X \in \mathbb{R}^{d\times d},
	\end{equation}
	where $\ast$ means {the} standard convolution with {one channel} and constant or periodic padding as in~\eqref{eq:defineConv}.
\end{theorem}
\begin{proof}
	Given the definition of $\widetilde{K}_{i,j}$ in \eqref{eq:decom_tildeK_1}, we need \text{only} verify that
	\begin{equation}\label{key}
		\widetilde K_{i,j} \ast X = P_{i,j} \ast S_{i,j} \ast X
	\end{equation}
	for any $i,j=-1,0,1$. For constant or periodic padding, we prove the above claim respectively.
	\paragraph{Periodic padding.} For this case, we notice that
	\begin{equation}\label{eq:periodic_S}
		\left[S_{i,j} \ast X \right]_{m,n} = X_{m+i,n+j}
	\end{equation}
	for any $i,j=-1,0,1$ and $1\le m,n\le d$. Therefore, we have
	\begin{equation}\label{key}
		\begin{split}
			\left[P_{i,j} \ast S_{i,j} \ast X\right]_{m,n} 
			&= \sum_{p,q = -k+1, \cdots, k-1} \left[ P_{i,j}\right]_{p,q} \left[ S_{i,j} \ast X\right]_{m+p, n+q} \\
			&= \sum_{p,q = -k+1, \cdots, k-1} \left[ P_{i,j}\right]_{p,q} \left[ X\right]_{m+p+i, n+q+j} \\
			&= \sum_{\substack {\widetilde p = -k+1+i, \cdots k-1+i,  \\ \widetilde q = -k+1+j, \cdots k-1+j}} \left[ P_{i,j}\right]_{\widetilde{p}-i, \widetilde{q}-j} \left[X\right]_{m + \widetilde{p}, n+\widetilde{q}},\\
			&=  \sum_{\widetilde{p}, \widetilde{q} = -k, \cdots, k}\left[ \widetilde{K}_{i,j}\right]_{\widetilde{p}, \widetilde{q}} \left[X\right]_{m + \widetilde{p}, n+\widetilde{q}} \\
			&= \left[\widetilde K_{i,j} \ast X\right]_{m,n}
		\end{split}
	\end{equation}
	for any $i,j=-1,0,1$ and $1\le m,n\le d$.
	\paragraph{Constant padding.} For this case, we split the proof into three cases according to different values of $|i|+|j|$.
	\begin{enumerate}
		\item $|i|+|j|=0$, i.e., $i=j=0$. Thus, for any $1 \le m,n \le d$, we have
		\begin{equation}\label{eq:split_1}
			\begin{split}
				\left[P_{0,0} \ast S_{0,0} \ast X \right]_{m,n} 
				&= \sum_{p,q = -k+1,\cdots,k-1} \left[ P_{0,0} \right]_{p,q} \left[ S_{0,0} \ast X \right]_{m+p,n+q} \\
				&= \sum_{p,q = -k+1,\cdots,k-1} \left[ K \right]_{p,q} \left[ X \right]_{m+p,n+q}\\
				&= \left[ \widetilde K_{0,0} \ast X \right]_{m,n}.
			\end{split}
		\end{equation}
		\item $|i|+|j|=2$, for example $(i,j) = (-1,-1)$ or $(1,-1)$. 
		Without loss of generality, let us consider the example $(i,j) = (1,-1)$ first. Thus, we have
		\begin{equation}\label{eq:split_21}
			\begin{split}
				&\sum_{p,q = -k+1,\cdots,k-1} \left[ P_{1,-1} \right]_{p,q} \left[ S_{1,-1} \ast X \right]_{m+p,n+q}  \\
				&= \left[ K \right]_{k,-k} \left[ S_{1,-1} \ast X \right]_{m+k-1,n-k+1}.
			\end{split}
		\end{equation}
		As there is padding for $S_{1,-1} \ast X$ when we calculate $P_{1,-1} \ast S_{1,-1} \ast X$, 
		it is necessary to compute $\left[ S_{1,-1} \ast X \right]_{m+k-1,n-k+1}$ carefully.
		By definition, we first have $\left[ S_{1,-1} \ast X \right]_{s,t}$ for $s,t=1:d$,
		\begin{equation}
			\left[ S_{1,-1} \ast X \right]_{s,t} = 
			\begin{cases}
				&a, \quad \text{ if }  s = d \text{ or } t=1, \\
				&\left[ X \right]_{s+1,t-1}, \quad \text{others}.
			\end{cases}
		\end{equation}
		We further mention that it is necessary to include padding in $S_{1,-1} \ast X$ in \eqref{eq:split_21}:
		\begin{equation}
			\left[ S_{1,-1} \ast X \right]_{m+k-1,n-k+1} = \begin{cases}
				&a, \quad \text{ if } m \ge d-k+2 \text{ or } n \le k-1, \\
				&\left[ S_{1,-1} \ast X \right]_{m+k-1,n-k+1}, \quad \text{others}.
			\end{cases}
		\end{equation}
		By combining the previous two equations and noticing that $k\ge2$, we can obtain that
		\begin{equation}\label{eq:constan_S1}
			\begin{aligned}
				&\left[ S_{1,-1} \ast X \right]_{m+k-1,n-k+1} \\
				= &\begin{cases}
					&a, \quad \text{ if } m \ge d-k+1 \text{ or }  n \le k, \\
					&\left[ X \right]_{m+k,n-k}, \quad \text{others},
				\end{cases} \\
				= &\left[ X \right]_{m+k,n-k}.
			\end{aligned}
		\end{equation}
		{Therefore}, we have
		\begin{equation}\label{eq:split_2}
			\begin{split}
				\left[P_{1,-1} \ast S_{1,-1} \ast X \right]_{m,n} &=\sum_{p,q = -k+1,\cdots,k-1} \left[ P_{1,-1} \right]_{p,q} \left[ S_{1,-1} \ast X \right]_{m+p,n+q}  \\
				&= \left[ K \right]_{k,-k} \left[ S_{1,-1} \ast X \right]_{m+k-1,n-k+1} \\
				&= \left[ K \right]_{k,-k} \left[ X \right]_{m+k,n-k} \\
				&=\left[ \widetilde K_{1,-1} \ast X \right]_{m,n}.
			\end{split}
		\end{equation}
		A similar derivation can be applied to the other three cases for $|i|+|j| = 2$.
		\item $|i|+|j| = 1$, for example, $(i,j) = (-1,0)$ or $(0,1)$.
		Without loss of generality, let us consider the example $(i,j) = (1,-1)$. 
		Thus, we have
		\begin{equation}
			\begin{split}
				&\sum_{p,q = -k+1,\cdots,k-1} \left[ P_{0,1} \right]_{p,q} \left[ S_{0,1} \ast X \right]_{m+p,n+q}\\
				= &\sum_{p = -k+1,\cdots,k-1} \left[ K \right]_{p,k} \left[ S_{0,1} \ast X \right]_{m+p,n+k-1}.
			\end{split}
		\end{equation}
		First, let us take $p>0$ and then compute $ \left[ S_{0,1} \ast X \right]_{m+p,n+k-1}$ in the same fashion.
		Thus, we have
		\begin{equation}
			\left[ S_{0,1} \ast X \right]_{s,t} = 
			\begin{cases}
				&a, \quad \text{ if }  t = d, \\
				&\left[ X \right]_{s,t+1}, \quad \text{others},
			\end{cases}
		\end{equation}
		and 
		\begin{equation}
			\left[ S_{0,1} \ast X \right]_{m+p,n+k-1} = \begin{cases}
				&a, \quad \text{ if } m \ge d-p+1 \text{ or } n \ge d-k+2, \\
				&\left[ S_{1,-1} \ast X \right]_{m+p,n+k-1}, \quad \text{others}.
			\end{cases}
		\end{equation}
		Furthermore, we can obtain that
		\begin{equation}\label{eq:constan_S2}
			\begin{aligned}
				&\left[ S_{0,1} \ast X \right]_{m+p,n+k-1} \\
				= &\begin{cases}
					&a, \quad \text{ if } m \ge d-p+1 \text{ or } n \ge d-k+1, \\
					&\left[ X \right]_{m+p,n+k}, \quad \text{others},
				\end{cases} \\
				= &\left[ X \right]_{m+p,n+k}.
			\end{aligned}
		\end{equation}
		For $p < 0$, we can also go through the previous steps to reach the same conclusion. Thus, we have
		\begin{equation}\label{eq:split_3}
			\begin{split}
				\left[P_{0,1} \ast S_{0,1} \ast X \right]_{m,n} &=\sum_{p,q = -k+1,\cdots,k-1} \left[ P_{0,1} \right]_{p,q} \left[ S_{0,1} \ast X \right]_{m+p,n+q}  \\
				&= \sum_{p = -k+1,\cdots,k-1} \left[ K \right]_{p,k} \left[ S_{0,1} \ast X \right]_{m+p,n+k-1} \\
				&= \sum_{p = -k+1,\cdots,k-1} \left[ K \right]_{p,k} \left[ X \right]_{m+p,n+k} \\
				&= \left[ \widetilde K_{0,1} \ast X \right]_{m,n}.
			\end{split}
		\end{equation}
		A similar derivation can be applied to {the} other three cases for $|i|+|j| = 1$.
	\end{enumerate}
	This completes the proof.
	
\end{proof}
According to the proof, the decomposition in~\eqref{eq:decomp_k} does not
hold for arbitrary {paddings} such as reflection or replication padding~\cite{paszke2019pytorch}, {because}
equations \eqref{eq:periodic_S}, \eqref{eq:constan_S1}, and \eqref{eq:constan_S2} can not be true.

By applying the above theorem {to decompose $P_{i,j}$} recursively, we have the following corollary.
\begin{corollary}Let  $K \in \mathbb{R}^{(2k+1)\times (2k+1)}$ be a large kernel with one channel and $d > k$.
	Then there exist $P_{(i_1,j_1),(i_2,j_2),\cdots, (i_{k-1},j_{k-1})} \in \mathbb{R}^{3\times3}$ 
	and $S_{i_m,j_m} \in \mathbb{R}^{3\times 3}$ for $i_m,j_m = -1,0,1$ and $m=1:k-1$  such that
	\begin{equation}\label{eq:decomp}
		K \ast X = \sum_{i_{k-1},j_{k-1}} \cdots \sum_{i_1,j_1}  P_{(i_1,j_1),\cdots, (i_{k-1},j_{k-1})}  \ast S_{(i_{k-1},j_{k-1})} \ast \cdots \ast S_{(i_1,j_1)} \ast X
	\end{equation}
	for any $X \in \mathbb{R}^{d\times d}$, where $\ast$ means the standard convolution with one channel as in~\eqref{eq:defineConv}.
\end{corollary}
\begin{proof}
	This can be {proved} by repeatedly applying Theorem~\ref{thm:decomp2k+1} for $P_{i,j}$ in \eqref{eq:decomp_k}  until each $P_{(i_1,j_1),(i_2,j_2),\cdots, (i_{k-1},j_{k-1})}$ becomes a $3\times 3$ kernel. 
\end{proof}

As mentioned in Remark~\ref{rem:PS}, we can collect all $P_{(i_1,j_1),(i_2,j_2),\cdots, (i_{k-1},j_{k-1})}$ into $P$ as a single convolution kernel with multi-channels. Therefore, the output channel of $P$ is $9^{k-1}$, which will be huge if $k$ is large.
Thanks to the special {pattern of zero in} $P_{i,j}$ in \eqref{eq:P_i_2k+1}, we have the following lemma to further reduce the number of non-zero {output} channels in $P$.
\begin{lemma}\label{thm:decom2}
	Let $K \in \mathbb{R}^{(2k+1)\times (2k+1)}$ and $d>k$.
	Then there is an index set 
	\begin{equation}\label{key}
		\bm I_{k-1} \subset \left\{  \left. \left( (i_1,j_1),\cdots,(i_{k-1},j_{k-1}) \right)~\right|~ i_m, j_m = \{-1,0,1\}, m=1:k-1 \right\}
	\end{equation}
	such that 
	\begin{equation}\label{key}
		K \ast X =  \sum_{\left((i_1,j_1), \cdots, (i_{k-1},j_{k-1})\right) \in \bm I_{k-1} }  P_{(i_1,j_1),\cdots, (i_{k-1},j_{k-1})} \ast S_{i_{k-1},j_{k-1}} \ast \cdots \ast S_{i_1,j_1} \ast X
	\end{equation}
	for any $X \in \mathbb{R}^{d\times d}$, where $\ast$ means {a} standard convolution with one channel. Moreover, we have the cardinality of $\bm I_{k-1}$ as
	\begin{equation}\label{key}
		\# \bm I_{k-1} = (2k-1)^2.
	\end{equation}
\end{lemma}
\begin{proof}This proof is based on the special distribution of zero for each $P_{i,j}$ in \eqref{eq:P_i_2k+1}. Assume that we have applied Theorem~\ref{thm:decomp2k+1} to $P_{i,j}$ for $n-1$-times with $n<k$, and obtained the following set of kernels
	\begin{equation}\label{key}
		\bm P_n := \left\{ \left. P_{(i_1,j_1),\cdots, (i_{n},j_{n})} ~\right|~  i_m, j_m =-1,0,1, m=1:n \right\}.
	\end{equation}
	It is easy to see that the cardinality of $\bm P_n$ is $9^n$. Here, we prove that the number of non-zero items in $\bm P_n$ is bounded by $(2n+1)^2$. Because of the special form of $P_{i,j}$ in \eqref{eq:P_i_2k+1}, we conclude that {for} non-zero $P_{(i_1,j_1),\cdots, (i_{n},j_{n})}$ there are only three types based on different {zero-patterns}. 
	\begin{enumerate}
		\item Type 1: Non-zero {items} on the corner. For example, $P_{-1,-1}$ and $P_{-1,1}$ for $n=1$, or $P_{(-1,-1),(-1,-1)}$ and $P_{(0,0),(1,-1)}$ for $n=2$. We denote the number of elements {with this type} as $C_n$.
		\item Type 2: Non-zero items on the boundary. For example, $P_{-1,0}$ and $P_{0,1}$ for $n=1$, or $P_{(0,-1),(0,-1)}$ and $P_{(0,0),(1,0)}$ for $n=2$. We denote the number of elements {with this type} as $B_n$.
		\item Type 3: Full kernel. For example, $P_{0,0}$ for $n=1$, or $P_{(0,0), (0,0)}$ for $n=2$. {A critical observation} is that there is only one item with this form in $\bm P_n$ for any $n$, i.e., $P_{(0,0), \cdots, (0,0)} \in \bm P_n$.
	\end{enumerate}
	The following rules describe the connections of the number of non-zero items between $\bm P_{n-1}$ and $\bm P_n$ when we apply Theorem~\ref{thm:decomp2k+1} to $\bm P_{n-1}$ in order to obtain $\bm P_n$.
	\begin{enumerate}
		\item Type 1:
		\begin{equation}\label{eq:C_n}
			C_n = C_{n-1} + 2B_{n-1} + 4,
		\end{equation}
		as each element in $\bm P_{n-1}$ with type 1 can make only one non-zero element in $\bm P_{n}$ with type 1, each element in $\bm P_{n-1}$ with type 2 can make two non-zero elements in $\bm P_{n}$ with type 1, and 
		each element in $\bm P_{n-1}$ with type 3 can make four non-zero elements in $\bm P_{n}$ with type 1.
		\item Type 2:
		\begin{equation}\label{eq:B_n}
			B_n = B_{n-1} + 4,
		\end{equation}
		as each element in $\bm P_{n-1}$ with type 2 can make one non-zero element in $\bm P_{n}$ with type 2, 
		each element in $\bm P_{n-1}$ with type 3 can make four non-zero elements in $\bm P_{n}$ with type 2, but each element in $\bm P_{n-1}$ with type 1 cannot make any non-zero element in $\bm P_{n}$ with type 2.
		\item Type 3: There is only one non-zero element in $\bm P_n$. First, this non-zero item cannot be produced from elements in $\bm P_{n-1}$ with either type 1 or type 2. In addition, each element in $\bm P_{n-1}$ with type 3 can make only one {non-zero} element in $\bm P_{n}$ with type 3. 
	\end{enumerate}
	According to the decomposition in Theorem~\ref{thm:decomp2k+1}, we have 
	\begin{equation}\label{key}
		C_1 = B_1 = 4
	\end{equation}
	as the initial {values} for \eqref{eq:C_n} and \eqref{eq:B_n}. Thus, we have
	\begin{equation}\label{key}
		C_n = 4n^2 \quad \text{and} \quad B_n = 4n,
	\end{equation}
	which means that the number of non-zero elements in $\bm P_n$ is
	\begin{equation}\label{key}
		C_n + B_n + 1 = 4n^2 + 4n + 1 = (2n+1)^2.
	\end{equation}
	Thus, the theorem is {proved} by taking $n=k-1$ and $\bm I_{k-1}$ as the index set of non-zero elements in $\bm P_{k-1}$. 
\end{proof}

By representing the previous theorem in terms of convolution with multi-channels globally, we obtain the following theorem.
\begin{theorem}\label{thm:decomp_k_gloabl}
	Let $K \in \mathbb{R}^{1 \times M \times (2k+1)\times (2k+1)}$ and $d>k$.
	Then there is a series of kernels $S^n \in \mathbb{R}^{c_{n-1} \times c_n \times 3 \times 3}$ with {multi-channels} and $P \in \mathbb{R}^{(2k-1)^2 \times M \times 3 \times 3}$ such that
	\begin{equation}\label{key}
		K \ast X =  P \ast S^{k-1} \ast S^{k-2} \ast \cdots \ast S^1 \ast X, \quad \forall X \in \mathbb{R}^{d\times d},
	\end{equation}
	where $c_n = (2n+1)^2$ for $n=1:k-1$ and $\ast$ means {the} standard convolution with multi-channels and padding as defined in~\eqref{eq:defineConv}. 
\end{theorem}

\begin{proof}
	First, we follow the proof in Theorem~\ref{thm:decomp2k+1} and notice that the index set $\bm I_{n}$ 
	is independent from the kernel $K$ and has this important feature:
	\begin{equation}\label{key}
		\left((i_1,j_1), (i_2,j_2), \cdots, (i_n,j_n)\right) \in \bm I_n \Rightarrow \left((i_1,j_1), (i_2,j_2), \cdots, (i_m,j_m)\right)  \in \bm I_m,
	\end{equation}
	if $m \le n$. Thus, we can define the following operator $\tau_n: \bm I_n \mapsto \bm I_{n-1}$ as
	\begin{equation}\label{key}
		\tau_n\left(\left((i_1,j_1), (i_2,j_2), \cdots, (i_n,j_n)\right)\right) = \left((i_1,j_1), (i_2,j_2), \cdots, (i_{n-1},j_{n-1})\right). 
	\end{equation}
	Then, for each $\bm I_n$, we fix a bijection
	\begin{equation}\label{key}
		\pi_n : \{ 1, 2, \cdots, (2n+1)^n \} \mapsto \bm I_n
	\end{equation}
	to give a unique position for each element in $\bm I_n$. For example, alphabetical order can be used.
	Thus, we construct $S^n \in \mathbb{R}^{c_{\ell-1} \times c_n \times 3 \times 3}$ by taking
	\begin{equation}\label{key}
		\left[ S^{n} \right]_{p,q} = \begin{cases}
			S_{i_n,j_n}, &\text{if} ~\pi_n(q) = 	\left((i_1,j_1), \cdots, (i_n,j_n)\right)\text{ and } \pi_{n-1}(p) = \tau_n(\pi_n(q)), \\
			0, \quad &\text{others},
		\end{cases}
	\end{equation}
	for all $n=1:k-1$.
	Therefore, we can check that
	\begin{equation}\label{key}
		\begin{aligned}
			\left[S^n \ast S^{n-1} \ast \cdots S^1 \ast X\right]_q 
			&= \sum_{p=1}^{c_{n-1}}[S^n]_{p,q} \ast [S^{n-1} \ast \cdots S^1 \ast X]_p \\
			&= S_{i_n,j_n} \ast [S^{n-1} \ast \cdots S^1 \ast X]_{ \pi_{n-1}^{-1}\left( \tau_n(\pi_n(q))\right)} \\
			&= S_{i_n,j_n} \ast S_{i_{n-1},j_{n-1}} \ast [S^{n-2} \ast \cdots S^1 \ast X]_{ \pi_{n-2}^{-1}\left(\tau_{n-1}( \tau_n(\pi_n(q)))\right)} \\
			&= \cdots \\
			&=S_{i_n,j_n} \ast S_{i_{n-1},j_{n-1}} \ast  \cdots \ast S_{i_{1},j_{1}}  \ast X,
		\end{aligned}
	\end{equation}
	where 
	\begin{equation}
		\left(  (i_1,j_1), (i_2,j_2), \cdots, (i_n,j_n) \right) = \pi_n(q) \in \bm I_n
	\end{equation}
	for all $1\le q \le (2n+1)^2$. 
	According to Theorem~\ref{thm:decomp2k+1}, for each channel $[K]_m \in \mathbb{R}^{(2k+1) \times (2k+1)}$ in $K\in \mathbb{R}^{1 \times M \times (2k+1) \times (2k+1)}$, we have
	\begin{equation}\label{eq:K_m_P}
		[K]_m \ast X = \sum_{\left((i_1,j_1), \cdots, (i_{k-1},j_{k-1})\right) \in \bm I_{k-1} }  P^m_{(i_1,j_1),\cdots, (i_{k-1},j_{k-1})} \ast S_{i_{k-1},j_{k-1}} \ast \cdots \ast S_{i_1,j_1} \ast X.
	\end{equation}
	Finally, we finish the proof by constructing $P \in \mathbb{R}^{(2k-1)^2 \times M \times 3 \times 3}$ as
	\begin{equation}\label{key}
		[P]_{p,m} = P^m_{\pi^{-1}_{k-1}(p)},
	\end{equation}
	where $P^m_{\pi^{-1}_{k-1}(p)}$ is defined in \eqref{eq:K_m_P}. 
\end{proof}

\section{Universal approximation theorem for classic CNNs}\label{sec:approx_cnn}
In this section, we show the universal approximation theorem for classic CNNs
with 2D image inputs {and} standard multi-channel convolutions.

First, let us introduce CNN architecture with input data $x \in \mathbb{R}^{d\times d}$ and
ReLU~\cite{nair2010rectified} activation function($\sigma(t) = {\rm ReLU}(t):=\max\{0,t\}$ for any $t\in \mathbb{R}$):
\begin{equation}\label{eq:cnn}
	\begin{cases}
		f^{\ell}(x) &= \sigma (K^\ell \ast f^{\ell-1}(x) + b^\ell \bm 1) \quad \ell = 1:L, \\
		f(x) &= a \cdot {\mathcal V}\left(f^{L}(x)\right),\\
	\end{cases}
\end{equation}
where $f^0(x) = x \in \mathbb{R}^{d\times d}$, $K^\ell \in \mathbb{R}^{c_{\ell-1} \times c_{\ell}\times 3 \times 3}$, $b^\ell \in \mathbb{R}^{c_\ell}$, $f^\ell \in \mathbb{R}^{c_\ell \times d\times d }$, $a\in \mathbb{R}^{c_Ld^2}$, and ${\mathcal V}(f^L(x))$ denotes the vectorization of $f^L(x) \in \mathbb{R}^{c_L \times d\times d}$ by taking
\begin{equation}\label{key}
	\left[{\mathcal V}\left(f^L(x)\right)\right]_{(c-1)d^2 + (s-1)d + t} = \left[f^L(x)\right]_{c,s,t}
\end{equation} 
for all $s,t = 1:d$ and $c = 1:c_L$. 
For simplicity, we extend the definition of ${\mathcal V}(\cdot)$ for {the} general tensor in $\mathbb{R}^{d\times d}$, $\mathbb{R}^{c_\ell \times d \times d}$, etc.
Here, $K^\ell \ast f^{\ell-1}(x)$ follows the definition of convolution with multi-channel and constant or periodic padding as in \eqref{eq:defineConv}. In addition, {we consider} the special form of bias in CNNs, 
\begin{equation}\label{key}
	b^\ell \bm 1 := \left([b^\ell]_{1} \mathbf{I}, [b^\ell]_{2} \mathbf{I}, \cdots, [b^\ell]_{c_\ell} \mathbf{I} \right) \in \mathbb{R}^{c_\ell \times d\times d},
\end{equation}
where $\mathbf{I} \in \mathbb{R}^{d\times d}$ with $[\mathbf{I}]_{s,t} = 1$ for all $s,t = 1:d$. Moreover, we notice that there is no pooling, subsampling, or coarsening operator (layer) to apply in the above CNN architecture. Furthermore, to investigate the approximation properties of CNNs on $\mathbb{R}^{d\times d}$, we consider $\mathbb{R}^{d\times d}$ as a $d^2$-dimensional vector space with Frobenius norm. 

Before we prove the main approximation theorem, let us introduce the next two
lemmas, which reveal the connection between deep {ReLU} CNNs and one-hidden-layer ReLU NNs.
\begin{lemma}\label{lemm:linear_decomp}
	For any $W \in \mathbb{R}^{N \times d^2}$ and $\alpha, \beta\in \mathbb{R}^N$, there is a convolutional kernel $K \in \mathbb{R}^{1\times N\times (2\lfloor d/2\rfloor + 1)\times (2\lfloor d/2\rfloor + 1)}$, bias $b \in \mathbb{R}^N$, and weight $a\in \mathbb{R}^{Nd^2}$ such that
	\begin{equation}\label{key}
		\alpha \cdot \sigma\left( W {\mathcal V}(x) + \beta \right) = a \cdot {\mathcal V}\left(\sigma \left( K \ast x + b \bm 1 \right) \right)
	\end{equation}
	for any $x\in \mathbb{R}^{d\times d}$.
\end{lemma}
\begin{proof}
	For simplicity, let us first assume that $d$ is odd. Then, we have $d = 2\lfloor d/2\rfloor + 1$ and
	\begin{equation}\label{key}
		\left[ K \ast x\right]_{n,\lceil d/2 \rceil, \lceil d/2 \rceil} = \sum_{s,t = -\lfloor d/2\rfloor}^{\lfloor d/2\rfloor} [K]_{n,s,t} [x]_{\lceil d/2 \rceil+s,\lceil d/2 \rceil+t} = {\mathcal V}([K]_n) \cdot {\mathcal V}(x).
	\end{equation}
	Thus, this proof is completed by taking $b = \beta$, 
	\begin{equation}\label{key}
		{\mathcal V}([K]_{n,:,:}) = [W]_{n,:}, \quad \text{and} \quad [a]_k = \begin{cases}
			[\alpha]_n, \quad &\text{if } k = (n-1)d^2 + \lceil d/2 \rceil^2,  \\
			0, \quad &\text{others},
		\end{cases}
	\end{equation}
	for all $n=1:N$. 
	If $d$ is even, we have $2\lfloor d/2\rfloor + 1 = d+1$. Thus, we can construct
	$a$ and $b$ as before and then take
	$ {\mathcal V}([K]_{n,- d/2:d/2-1,-d/2:d/2-1}) =  [W]_{n,:}$
	and $[K]_{n,d/2,-d/2:d/2} = [K]_{n,-d/2:d/2,d/2} = 0$. Thus, we have
	\begin{equation}\label{key}	
		\begin{aligned}
			\left[ K \ast x\right]_{n, d/2+1,  d/2 +1} 
			&= \sum_{s,t = - d/2}^{ d/2} [K]_{n,s,t} [x]_{ d/2+1 +s, d/2+1 +t} \\
			&= {\mathcal V}([K]_{n,- d/2:d/2-1,-d/2:d/2-1}) \cdot {\mathcal V}(x),
		\end{aligned}
	\end{equation}
	which finishes the proof.
\end{proof}
Basically, the above lemma shows that a {ReLU NN} function with one hidden layer can be represented by a one-layer CNN with a large kernel.
\begin{lemma}\label{lemm:nonl_big_K}
	For any bounded set $\Omega \subset \mathbb{R}^{d\times d}$, kernel $K \in \mathbb{R}^{1\times N\times (2\lfloor d/2\rfloor + 1)\times (2\lfloor d/2\rfloor + 1)}$, and bias vector $b \in \mathbb{R}^N$, there is a series of kernels $K^\ell \in \mathbb{R}^{c_{\ell-1} \times c_\ell \times 3 \times 3}$ and bias vectors $b^\ell \in \mathbb{R}^{c_\ell}$ such that
	\begin{equation}\label{eq:Kx+b_sigma}
		\left[K\ast x + b \bm1 \right]_{n,\lceil d/2 \rceil, \lceil d/2 \rceil} = \left[ K^{\lfloor d/2\rfloor} \ast f^{\lfloor d/2\rfloor-1}(x) + b^{\lfloor d/2\rfloor} \bm 1\right]_{n,\lceil d/2 \rceil, \lceil d/2 \rceil},~ \forall x \in \Omega,
	\end{equation}
	where $f^0(x) = x$, $c_\ell = (2\ell+1)^2$ for $\ell=1:{\lfloor d/2\rfloor}-1$, $c_{{\lfloor d/2\rfloor}}= N$, and
	\begin{equation}\label{key}
		f^{\ell}(x) = \sigma \left( K^\ell \ast f^{\ell-1}(x) + b^\ell \bm 1  \right).
	\end{equation}
\end{lemma}
\begin{proof}According to Theorem~\ref{thm:decomp_k_gloabl}, we know there is $P\in \mathbb{R}^{(2\lfloor d/2\rfloor - 1)^2\times N \times 3 \times 3}$
	and $S^\ell \in  \mathbb{R}^{c_{\ell-1} \times c_\ell \times 3 \times 3}$ with $c_\ell = (2\ell+1)^2$ for $\ell=1:{\lfloor d/2\rfloor}-1$ such that
	\begin{equation}\label{key}
		K\ast x = P\ast S^{\lfloor d/2\rfloor - 1} \ast \cdots \ast S^{1} \ast x
	\end{equation}
	for any $x \in \Omega$. Thus, we can take
	\begin{equation}\label{key}
		K^\ell = S^\ell, \quad \ell=1:{\lfloor d/2\rfloor}-1 \quad \text{and} \quad K^{\lfloor d/2\rfloor} = P.
	\end{equation}
	Then, we can prove \eqref{eq:Kx+b_sigma} by taking $b^\ell$ in an appropriate way. From $\ell=1$ to ${\lfloor d/2\rfloor}-1$, we define $b^\ell$ consecutively as
	\begin{equation}\label{key}
		[b^\ell]_q =  \max_{1 \le s, t \le d}\sup_{x\in \Omega}\left| \left[K^\ell \ast f^{\ell-1}(x)\right]_{q,s,t} \right|, \quad q = 1:c_\ell.
	\end{equation}
	As $\Omega \subset \mathbb{R}^{d\times d}$ is bounded and $\sigma = {\rm ReLU}$ is continuous, we know that
	\begin{equation}\label{key}
		[b^\ell]_q < \infty, \quad \forall \ell = 1:{\lfloor d/2\rfloor}-1, \quad q = 1:c_\ell.
	\end{equation}
	Therefore, we have
	\begin{equation}\label{key}
		f^{\ell}(x) = \sigma \left( K^\ell \ast f^{\ell-1}(x) + b^\ell \bm 1  \right) =  K^\ell \ast f^{\ell-1}(x) + b^\ell \bm 1
	\end{equation}
	because of the definition of $\sigma(x)$ and $b^\ell$. Thus, we have
	\begin{equation}\label{key}
		K^{\lfloor d/2\rfloor} \ast f^{\lfloor d/2\rfloor-1}(x) = P\ast S^{\lfloor d/2\rfloor - 1} \ast \cdots \ast S^{1} \ast x + B = K\ast x + B,
	\end{equation}
	where
	\begin{equation}
		B = \sum_{\ell=2}^{\lfloor d/2\rfloor-1} P \ast S^{\lfloor d/2\rfloor - 1} \ast \cdots \ast S^{\ell} \ast (b^{\ell-1}\bm 1)  + P \ast (b^{\lfloor d/2\rfloor - 1}\bm 1) \in \mathbb{R}^{N \times d\times d},
	\end{equation}
	which is constant in respect to $x$. 
	Finally, we take 
	\begin{equation}\label{key}
		[b^{\lfloor d/2\rfloor}]_n = [b]_n - [B]_{n,\lceil d/2 \rceil, \lceil d/2 \rceil},
	\end{equation}
	which finishes the proof. 
\end{proof}
Lemma~\ref{lemm:nonl_big_K} shows that a one-layer ReLU CNN with a large kernel
{can} be represented by a deep ReLU CNN with multi-channel $3\times 3$ kernels. 

To obtain our final theorem (Theorem~\ref{thm:approx_CNN}), let us first recall the following approximation result for one-hidden-layer ReLU NNs.
\begin{theorem}[\cite{bach2017breaking,siegel2021improved}]\label{thm:approx_fN}
	Assume {that} $f: \Omega \subset \mathbb{R}^D \mapsto \mathbb{R}$ and {that} $\Omega$ is bounded. Then there is a {ReLU NN} with one hidden layer $f_N(x) = \alpha \cdot \sigma\left( W x + \beta \right)$ where
	$W \in \mathbb{R}^{N \times D}$ and $\alpha, \beta\in \mathbb{R}^N$, such that
	\begin{equation}\label{eq:approx_DNN}
		\| f - f_N \|_{L^2(\Omega)} \lesssim  N^{-\frac{1}{2}-\frac{3}{2D}} \|f\|_{ \mathscr{K}_1(\mathbb D)}.
	\end{equation}
\end{theorem}
Here, $a\lesssim b$ means $a \le Cb$ where $C$ depends only on dimension $D$ and domain $\Omega$.
{In addition}, $\|f\|_{\mathscr{K}_1(\mathbb D)}$ is the norm defined by the {gauge} of $B_1(\mathbb D)$,
\begin{equation}\label{key}
	\|f\|_{\mathscr{K}_1(\mathbb D)} = \inf \{ c > 0 ~:~ f \in c B_1(\mathbb D)\},
\end{equation} 
where $B_1(\mathbb D)$ is given by
\begin{equation}\label{key}
	B_1 (\mathbb D) = \overline{\left\{ \sum_{j=1}^n a_j h_j~:~ n\in \mathbb N, h_j \in \mathbb{D}, \sum_{j=1}^n |a_j| \le 1\right\}},
\end{equation}
and 
\begin{equation}\label{key}
	\mathbb D = \{ \sigma(\omega \cdot x + b) ~:~ \omega \in \mathbb{R}^D, b \in \mathbb{R} \}
\end{equation}
is the dictionary generated by the activation function $\sigma (t) = {\rm ReLU}(t)$.
More details about the $\|f\|_{\mathscr{K}_1(\mathbb D)}$ norm, its approximation properties, and its connections with what is known as Barron norm can be found in~\cite{siegel2021characterization,siegel2021improved,e2021barron}. 
Generally, the underlying model for image classification is a piecewise constant function for which it is impossible to have a finite $\mathscr{K}_1(\mathbb D)$ norm. However, the ReLU CNN functions that we discuss in this paper are the feature extraction parts of ReLU CNN models for image classification. Thus, $f(x)$ may have a  finite $\mathscr{K}_1(\mathbb D)$ norm as a feature extraction function not a classification model.

By combining Lemma~\ref{lemm:linear_decomp}, Lemma~\ref{lemm:nonl_big_K}, and Theorem~\ref{thm:approx_fN}, we have the following approximation theorem of deep CNNs with multi-channel $3\times3$ kernels.
\begin{theorem}\label{thm:approx_CNN}
	Let $\Omega \subset \mathbb{R}^{d\times d}$ be bounded. Assume that $f: \Omega \mapsto \mathbb{R}$ and $\|f\|_{\mathscr{K}_1(\mathbb D)} < \infty$.
	Then there is a CNN function $\widetilde f: {R}^{d\times d}  \mapsto \mathbb{R}$ as defined in \eqref{eq:cnn}
	with multi-channel $3\times 3$ kernels, where
	\begin{equation}
		\begin{aligned}
			&\text{depth (number of convolutional layers):} \quad &L& = \lfloor d/2\rfloor, \\
			&\text{width (number of channels at each layer):} \quad &c_\ell& = (2\ell + 1 )^2,
		\end{aligned}
	\end{equation}
	for $\ell = 1:L-1$ and $c_L = N$, such that
	\begin{equation}\label{key}
		\|f - \widetilde f\|_{L^2(\Omega)}   \lesssim  N^{-\frac{1}{2}-\frac{3}{2d^2}}\|f\|_{\mathscr{K}_1(\mathbb D)}.
	\end{equation}
\end{theorem}
\begin{proof}
	First, we assume that $f_N({\mathcal V}(x))$ is the approximation of $f(x)$ using a one-hidden-layer ReLU NN, as shown in Theorem~\ref{thm:approx_fN}. According to Lemma~\ref{lemm:linear_decomp} and Lemma~\ref{lemm:nonl_big_K}, there is a CNN $\widetilde f(x)$ as defined in~\eqref{eq:cnn} with the hyperparameters listed above such that $\widetilde f(x) = f_N({\mathcal V}(x))$ for any $x\in \Omega$. This completes the proof. 
\end{proof}

Here, we notice that the total number of free parameters in $f_N({\mathcal V}(x))$, as shown in Theorem~\ref{thm:approx_fN}, is $\mathcal N_F = N(d^2 + 2)$ if $x\in \mathbb{R}^{d\times d}$. We also note that the total number of free parameters of the CNN function $\widetilde f(x)$ as in~\eqref{eq:cnn} with {hyperparameters} in Theorem~\ref{thm:approx_CNN} is 
\begin{equation}\label{key}
	\begin{aligned}
		\mathcal N_C = &\sum_{\ell=1}^{\lfloor d/2\rfloor-1}	\left( \underbrace{9(2\ell+1)^2(2\ell-1)^2}_{K^\ell}+ \underbrace{(2\ell+1)^2}_{b^\ell} \right) \\
		&+ \underbrace{N\left( (2\lfloor d/2\rfloor-1)^2 +1\right)}_{K^{\lfloor d/2\rfloor}~\&~b^{\lfloor d/2\rfloor}} + \underbrace{Nd^2}_a \\
		\le &~2(d^5 + Nd^2).
	\end{aligned}
\end{equation}
A Comparison of $\mathcal N_F$ and $\mathcal N_C$ shows that the 
convolutional neural networks with hyperparameters (depth $L$ and width $c_\ell$) in Theorem~\ref{thm:approx_CNN} can achieve the same asymptotic approximation order of one-hidden-layer ReLU NNs as $\mathcal O\left(N^{-\frac{1}{2}-\frac{3}{2d^2}}\right)$.  That is, there is an upper bound for the approximation error as $CN^{-\frac{1}{2}-\frac{3}{2d^2}}$ where $C$ depends only on dimension $d$ and domain $\Omega$.

In addition, in Theorem~\ref{thm:approx_CNN}, to achieve the approximation power for ReLU CNNs, it is 
necessary for the depth to exceed $d/2$ and for the number of layers in $\ell$-th layer to be at least $(2\ell+1)^2$. However, it is by no means common for practical CNN models to meet both of these conditions.
Here, we may interpret these conditions from other perspectives. For the depth (number of layers) of CNN models for image classification, we usually apply ResNet~\cite{he2016deep} CNNs with 18 or 34 layers for CIFAR-10 and CIFAR-100 with input images in $\mathbb{R}^{32\times32}$. Furthermore, we prefer deeper ResNet CNNs, for example, ResNet~\cite{he2016identity} with 50, 101, or more than 1,000 layers for ImageNet, which has input images in $\mathbb{R}^{224\times224}$. In all these examples, depth $L$ is greater than $d/2$ for input images in $\mathbb{R}^{d\times d}$. For the width (number of channels) of CNN models for image classification, we notice that every practical CNN model increases the input channel to a relatively large number and retains this width for several layers. These two observations indicate that it is necessary to include more layers and channels in practical CNN models.

Moreover, Theorem~\ref{thm:approx_CNN} requires a huge number of channels in the last layer to achieve a small approximate error. However, the number of channels in the last layer in practical CNNs is not very large in general. For example, one may take 512, 1,024, or 2,048 channels in the output layer for CIFAR and ImageNet classification problems. 
To understand this, we recall that the target function in this approximation result is the feature extraction function not the piecewise constant classification function. Thus, the $\|f\|_{\mathscr{K}_1(\mathbb D)}$ norm may be very small such that a relatively small $N$ may be enough to achieve sufficient approximation power. This implies that the feature extraction functions in image classification may lie in a special function class which is much smaller than $\mathscr{K}_1(\mathbb D)$.

Furthermore, Theorem~\ref{thm:approx_CNN} does not tell us why CNNs are much better than DNNs for image classification in terms of approximation power. However, Theorem~\ref{thm:approx_CNN} does indicate that CNNs with certain structures are no worse than DNNs in terms of approximation. This is important when CNNs are applied in fields in which approximation accuracy is a critical metric, such as numerical solutions of PDEs~\cite{karniadakis2021physics}. On the other hand, Theorem~\ref{thm:approx_CNN} shows that deep ReLU CNNs with certain structures can represent ReLU NNs with one hidden layer. This implies that the function class of ReLU CNNs is much larger than the function class of one-hidden-layer ReLU NNs. More precisely, the function class of ReLU CNNs may include or can efficiently approximate some very special functions that cannot be approximated directly using ReLU DNNs.

\section{Approximation properties of ResNet and MgNet}\label{sec:approx_mgnet}
In this section, we show the approximation properties for one version of ResNet~\cite{he2016deep}, pre-act ResNet~\cite{he2016identity}, and MgNet~\cite{he2019mgnet}.

\paragraph{Approximation properties of ResNet and pre-act ResNet. } First, let us introduce some mathematical formulas to define these two network functions:
\begin{description}
	\item[ResNet:] 
	\begin{equation}\label{eq:resnet}
		\begin{cases}
			f^{\ell}(x) &=\sigma\left( R^\ell \ast f^{\ell-1}(x) +  B^\ell \ast \sigma \left(A^\ell\ast f^{\ell-1}(x) + a^\ell \bm 1\right)+ b^\ell \bm 1\right),\\
			f(x) &= a \cdot {\mathcal V}\left(f^{L}(x)\right),
		\end{cases}
	\end{equation}
	for $\ell = 1:L$, where $f^0(x) = x \in \mathbb{R}^{d\times d}$, $A^\ell \in \mathbb{R}^{c_{\ell-1}\times C_{\ell} \times 3 \times 3}$, $B^\ell \in \mathbb{R}^{C_{\ell}\times c_{\ell} \times 3 \times 3}$, $R^\ell \in \mathbb{R}^{c_{\ell-1} \times c_{\ell} \times 1 \times 1}$, $a^\ell \in \mathbb{R}^{C_\ell}$, $b^{c_\ell}$, and $a\in \mathbb{R}^{c_Ld^2}$.
	\item[Pre-act ResNet:] 
	\begin{equation}\label{eq:presnet}
		\begin{cases}
			f^{\ell}(x) &= R^\ell \ast f^{\ell-1}(x) +  \sigma \left( B^\ell \ast \sigma \left( A^\ell\ast f^{\ell-1}(x) + a^\ell \bm 1\right) + b^\ell \bm 1\right), \\
			f(x) &= a \cdot {\mathcal V}\left(f^{L}(x)\right),
		\end{cases}
	\end{equation}
	for $\ell = 1:L$, where $f^0(x) = x \in \mathbb{R}^{d\times d}$, $A^\ell \in \mathbb{R}^{c_{\ell-1}\times C_{\ell} \times 3 \times 3}$, $B^\ell \in \mathbb{R}^{C_{\ell}\times c_{\ell} \times 3 \times 3}$, $R^\ell \in \mathbb{R}^{c_{\ell-1} \times c_{\ell} \times 1 \times 1}$, $a^\ell \in \mathbb{R}^{C_\ell}$, $b^{c_\ell}$, and $a\in \mathbb{R}^{c_Ld^2}$.
\end{description}
Each iteration from $f^{\ell-1}(x)$ to $f^\ell(x)$ is called a basic block of ResNet or pre-act ResNet. Note here that we {add} an extra $1\times 1$ kernel in front of $f^{\ell-1}(x)$ in each block, introduced in~\cite{he2016deep} {originally}, to adjust {to} the change of {channels} from $f^{\ell-1}(x)$ to $f^{\ell}(x)$. 
In addition, we notice that ResNet and pre-act ResNet can degenerate {into a} classical CNN as in~\eqref{eq:cnn}, if
we take $R^\ell = 0$. In this sense, ResNet and pre-act ResNet have approximation properties if we assume that $R^{\ell}$ is not given a priori but set as a trainable parameter. 
However, a key reason for the success of ResNet is that we set $R^\ell$ a priori, especially when taking $R^\ell$ as {the} identity operator when $f^{\ell-1}(x)$ and $f^{\ell}(x)$ have the same number of channels~\cite{he2016deep,he2016identity}. Thus, we consider ResNet and pre-act ResNet networks with an arbitrarily given value for $R^\ell$.

\begin{theorem}\label{thm:approx_ResNet}
	Let $\Omega \subset \mathbb{R}^{d\times d}$ be bounded, and assume that $f: \Omega \mapsto \mathbb{R}$ with  $\|f\|_{ \mathscr{K}_1(\mathbb D)}< \infty$. 
	For any given $R^\ell \in\mathbb{R}^{c_{\ell-1}\times c_{\ell} \times1\times1}$, there is a ResNet network
	$\widetilde f(x)$ as in~\eqref{eq:resnet}, where the hyperparameters satisfy
	\begin{equation}
		\begin{aligned}
			\text{depth (number of  blocks):} \quad &L = \lfloor \lfloor d/2\rfloor/2\rfloor, \\
			\text{width (number of channels):} \quad &c_\ell = (4\ell + 1 )^2,~C_\ell =2(4\ell-1)^2 
		\end{aligned}
	\end{equation}
	for $\ell=1:L-1$, $C_L =2(4L-1)^2$, and $c_L = N$, such that
	\begin{equation}\label{key}
		\|f - \widetilde f\|_{L^2(\Omega)}   \lesssim  N^{-\frac{1}{2}-\frac{3}{2d^2}} \|f\|_{ \mathscr{K}_1(\mathbb D)}.
	\end{equation}
\end{theorem}

\begin{proof}
	Without loss of generality, let us assume that $\frac{d}{4}$ is an integer. If not, some zero boundary layers
	can be added to enlarge the dimension of the original data to satisfy this condition. 
	First, we assume a CNN function $\widehat f(x) = f_N({\mathcal V}(x)) = \alpha \cdot \sigma\left(W{\mathcal V}(x) + \beta \right) $ defined in \eqref{eq:cnn}, which approximates $f(x)$, as described in Theorem~\ref{thm:approx_CNN}:
	$$
	\|f - \widehat f\|_{L^2(\Omega)} \lesssim N^{-\frac{1}{2} - \frac{3}{2d^2}}  \|f\|_{ \mathscr{K}_1(\mathbb D)}.
	$$
	Then, we can construct $\widetilde{f}(x)$ using kernels $K^\ell$ for $\ell=1:d/2$, $\widehat{b}^{d/2}$ and $\widehat a$ in $\widehat{f}(x)$. We define 
	\begin{equation}\label{key}
		A^\ell = \begin{pmatrix}
			K^{2\ell-1} \\
			I \\
			0
		\end{pmatrix}
		, \quad [a^\ell]_q = \max_{m,n}\sup_{x\in \Omega}\left|\left[A^\ell\ast f^{\ell-1}(x) \right]_{q,m,n}\right|, \quad \ell = 1:L,
	\end{equation}
	where $I$ is the identity kernel and $0$ is used to pad the output channel to be $C_\ell$.
	That is,
	\begin{equation}\label{key}
		\sigma\left(A^\ell \ast f^{\ell-1}(x) + a^\ell \bm 1\right) = \begin{pmatrix}
			K^{2\ell-1}\ast f^{\ell-1} \\
			f^{\ell-1} \\
			0
		\end{pmatrix} + a^\ell \bm 1.
	\end{equation}
	Moreover, we define
	\begin{equation}
		B^\ell = 
		\begin{pmatrix} 
			K^{2\ell}, -R^{\ell}, 0
		\end{pmatrix}, \quad \ell  =1:L
	\end{equation}
	and 
	\begin{equation}\label{key}
		\quad [b^\ell]_q = \max_{1\le s,t \le d}\sup_{x\in \Omega}\left| \left[R^\ell \ast f^{\ell-1}(x) +  B^\ell \ast \left(A^\ell\ast f^{\ell-1}(x) + a^\ell \bm 1\right) \right]_{q,s,t} \right|
	\end{equation}
	for  $\ell = 1:L-1$.
	Given the definition of $b^\ell$,  it follows that \\
	$\left[ R^\ell \ast f^{\ell-1}(x) + B^\ell \ast \sigma \left(A^\ell\ast f^{\ell-1}(x) + a^\ell \bm 1\right) + b^\ell \bm 1 \right]_{q,s,t} \ge 0$. Given the definition of $\sigma(t) = \rm{ReLU}(t) := \max\{0,t\}$, we have
	\begin{equation}\label{key}
		\begin{aligned}
			f^\ell(x) &= \sigma \left( R^\ell \ast f^{\ell-1}(x) + B^\ell \ast \sigma \left(A^\ell\ast f^{\ell-1}(x) + a^\ell \bm 1\right) + b^\ell \bm 1 \right)\\
			&= R^\ell \ast f^{\ell-1}(x) + B^\ell \ast \left(A^\ell\ast f^{\ell-1}(x) + a^\ell \bm 1\right)  + b^\ell \bm 1\\
			&= R^\ell \ast f^{\ell-1}(x) + K^{2\ell} \ast K^{2\ell-1} \ast f^{\ell-1}(x) - R^\ell\ast f^{\ell-1}(x) + B^\ell \ast a^\ell \bm 1 + b^\ell \bm 1\\
			&=K^{2\ell} \ast K^{2\ell-1} \ast f^{\ell-1}(x)+ B^\ell \ast a^\ell \bm 1 + b^\ell \bm 1
		\end{aligned}
	\end{equation}
	for $\ell=1:L-1$. By taking $\ell =L$, we have
	\begin{equation}\label{key}
		\begin{aligned}
			f^L(x) &= R^L \ast f^{L-1}(x) + B^L \ast \sigma \left(A^L\ast f^{L-1}(x) + a^L \bm 1\right) +b^L \bm 1\\
			&= K^{2L} \ast K^{2L-1} \ast f^{L-1}(x) + K^{2L}\ast a^L \bm 1 + b^L \bm 1 \\
			&= K^{d/2} \ast K^{d/2-1} \ast \cdots \ast K^{1} \ast x + B,
		\end{aligned}
	\end{equation}
	where $B$ is a constant tensor, which is similar to the case presented in Lemma~\ref{lemm:nonl_big_K}. 
	This suggests $b^L$ defined as
	\begin{equation}\label{key}
		[b^L]_n = {[ \beta]_n} - [B]_{n,d/2,d/2},
	\end{equation}
	where $\beta \in \mathbb{R}^N$ {proceeds from} $\widehat f(x) = f_N( \mathcal V (x)) = \alpha \cdot \sigma\left(W \mathcal V(x) + \beta \right)$ for any $x\in \Omega$. 
	Thus, it follows that
	\begin{equation}\label{key}
		f^{L}(x) = \widehat f^L(x),
	\end{equation}
	and we complete the proof by taking $a = \widehat{a}$. 
\end{proof}

\begin{theorem}\label{thm:approx_pResNet}
	Let $\Omega \subset \mathbb{R}^{d\times d}$ be bounded, and assume that $f: \Omega \mapsto \mathbb{R}$ with $ \|f\|_{ \mathscr{K}_1(\mathbb D)}$. 
	For any given $R^\ell \in\mathbb{R}^{c_{\ell-1}\times c_{\ell} \times1\times1}$, there is a pre-act ResNet network $\widetilde f(x)$ as in~\eqref{eq:presnet}, where the hyperparameters satisfy
	\begin{equation}
		\begin{aligned}
			&\text{depth (number of  blocks):} \quad &L& = \lfloor \lfloor d/2\rfloor/2\rfloor, \\
			&\text{width (number of channels):} \quad &c_\ell& = (4\ell + 1 )^2,~C_\ell =2(4\ell-1)^2
		\end{aligned}
	\end{equation}
	for $\ell=1:L-1$, $C_L =2(4L-1)^2$ and $c_L = N+2$, such that 
	\begin{equation}\label{key}
		\|f - \widetilde f\|_{L^2(\Omega)}   \lesssim  N^{-\frac{1}{2}-\frac{3}{2d^2}} \|f\|_{ \mathscr{K}_1(\mathbb D)}.
	\end{equation}
\end{theorem}

\begin{proof}
	We still assume that $\frac{d}{4}$ is an integer. Then, we construct pre-act ResNet as in~\eqref{eq:presnet}
	with {a} similar structure as {that described} in Theorem~\ref{thm:approx_ResNet}.
	According to Lemma~\ref{lemm:linear_decomp} and Lemma~\ref{lemm:nonl_big_K}, we notice that
	kernels $K^\ell$ for $\ell=1:\lfloor d/2\rfloor - 1$ as described in Theorem~\ref{thm:approx_CNN} are independent from $f(x)$. More precisely, these kernels can be defined by using $S^\ell$ presented in Theorem~\ref{thm:decomp_k_gloabl}. Thus, we take
	\begin{equation}\label{key}
		A^\ell = \begin{pmatrix}
			K^{2\ell-1} \\
			I \\
			0
		\end{pmatrix}
		, \quad [a^\ell]_q = \max_{m,n}\sup_{x\in \Omega}\left|\left[A^\ell\ast f^{\ell-1}(x) \right]_{q,m,n}\right|
	\end{equation}
	for $\ell = 1:L$ where $I$ is the identity kernel and $0$ is used to pad the output channel to be $C_\ell$.
	That is,
	\begin{equation}\label{key}
		\sigma\left(A^\ell \ast f^{\ell-1}(x) + a^\ell \bm 1\right) = \begin{pmatrix}
			K^{2\ell-1}\ast f^{\ell-1} \\
			f^{\ell-1} \\
			0
		\end{pmatrix} + a^\ell \bm 1.
	\end{equation}
	Moreover, we define
	\begin{equation}
		B^\ell = 
		\begin{pmatrix} 
			K^{2\ell}, -R^{\ell}, 0
		\end{pmatrix}, \quad \ell  =1:L-1
	\end{equation}
	and 
	\begin{equation}\label{key}
		\quad [b^\ell]_q = \max_{1\le s,t \le d}\sup_{x\in \Omega}\left|\left[B^\ell \ast \left(A^\ell\ast f^{\ell-1}(x) + a^\ell \bm 1\right) \right]_{q,s,t}\right|, \quad \ell = 1:L-1.
	\end{equation}
	This means {that}
	\begin{equation}\label{key}
		\begin{aligned}
			f^\ell(x)&= R^\ell \ast f^{\ell-1}(x) +\sigma\left( B^\ell \ast \sigma \left(A^\ell\ast f^{\ell-1}(x) + a^\ell \bm 1\right) + b^\ell \bm 1 \right)\\
			&= R^\ell \ast f^{\ell-1}(x) + B^\ell \ast \left(A^\ell\ast f^{\ell-1}(x) + a^\ell \bm 1 \right) + b^\ell \bm 1 \\
			&= R^\ell \ast f^{\ell-1}(x) + K^{2\ell} \ast K^{2\ell-1} \ast f^{\ell-1}(x) - R^\ell\ast f^{\ell-1}(x) + B^\ell \ast a^\ell \bm 1 + b^\ell \bm 1\\
			&= K^{2\ell} \ast K^{2\ell-1} \ast f^{\ell-1}(x)+ B^\ell \ast a^\ell \bm 1 + b^\ell \bm 1
		\end{aligned}
	\end{equation}
	for $\ell=1:L-1$. 
	
	Next, we show how to define $B^L$. First, let us denote $f_N({\mathcal V}(x)) = \alpha \cdot \sigma \left( Wx + \beta\right)$ {as} the fully connected neural network approximation for $f(x)$ as in Theorem~\ref{thm:approx_fN}. In addition, we denote $\overline f(x) = \overline a \cdot \overline{f}^{d/2}(x) = f_N(x)$ {as} the CNN approximation of $f(x)$ as in Theorem~\ref{thm:approx_CNN}. Now, we can take 
	\begin{equation}\label{key}
		[a]_n =  \begin{cases} [\overline a]_n, \quad &\text{ if } n \le Nd^2,\\
			-1, \quad &\text{ if } n = Nd^2 + (d/2)^2 \text{ or } (N+1)d^2 + (d/2)^2, \\
			0, \quad &\text{ others}.
		\end{cases}
	\end{equation}
	{Recalling} that $f^{L-1}(x)$ in the  pre-act ResNet now is a linear function, we have
	\begin{equation}\label{key}
		a \cdot {\mathcal V}	\left( \left[R^L \ast f^{L-1}(x)\right] \right) = h\cdot {\mathcal V}(x) + c,
	\end{equation}
	where $h \in \mathbb{R}^{d^2}$ and $c\in \mathbb{R}$. 
	Then, we can redefine a new  fully connected neural network function
	\begin{equation}\label{key}
		f_{N+2} := \alpha \cdot \sigma \left( W{\mathcal V}(x) + \beta\right) + \sigma(h\cdot {\mathcal V}(x) + c) + \sigma( - h\cdot {\mathcal V}(x) - c).
	\end{equation}
	According to Lemma~\ref{lemm:nonl_big_K}, we have a CNN function $\widehat{f}(x)$ such that
	\begin{equation}\label{key}
		[\widehat{f}^{d/2}(x)]_{:,d/2, d/2}  = \left[ \sigma (\widehat K^{d/2} \ast\widehat  f^{d/2-1} + \widehat b^\ell) \right]_{:,d/2, d/2}= \begin{pmatrix}
			\sigma \left( W{\mathcal V}(x) + \beta \right) \\
			\sigma(h\cdot {\mathcal V}(x) + c)\\
			\sigma( - h\cdot {\mathcal V}(x) - c)
		\end{pmatrix}.
	\end{equation}
	As noticed before, $K^\ell$ for $\ell=1:d/2-1$ have the same structure in Lemma~\ref{lemm:nonl_big_K}. As a result, we have
	\begin{equation}\label{key}
		\widehat{f}^{d/2-1}(x) = \overline f^{d/2-1}(x), \quad \forall x \in \Omega.
	\end{equation}
	Then, we take 
	\begin{equation}\label{key}
		K^{L} = \widehat{K}^{d/2},
	\end{equation}
	which implies that
	\begin{equation}\label{key}
		\begin{aligned}
			f^L(x) &= R^L \ast f^{L-1}(x) + \sigma\left( B^L \ast \sigma \left(A^L\ast f^{L-1}(x) + a^L \bm 1\right) + b^L \bm 1 \right)\\
			&= R^L \ast f^{L-1}(x) + \sigma \left(K^{2L} \ast K^{2L-1} \ast f^{L-1}(x) + B  + b^L \bm 1\right)\\
			&= R^L \ast f^{L-1}(x) + \sigma \left(\widehat K^{d/2} \ast K^{d/2-1} \ast \cdots \ast K^{1}\ast x + B  + b^L \bm 1\right),
		\end{aligned}
	\end{equation}
	where $B$ is a constant tensor similar to the case described in Theorem~\ref{thm:approx_ResNet}. Thus, we can take
	\begin{equation}\label{key}
		[b^L]_n = \begin{cases}
			[\beta]_n - [B]_{n,d/2,d/2}, \quad &\text{ if } n \le N, \\ 
			c - [B]_{N+1,d/2,d/2},\quad &\text{ if } n  = N+1, \\ 
			-c - [B]_{N+2,d/2,d/2}, \quad &\text{ if } n = N+2,
		\end{cases}
	\end{equation} 
	which leads to
	\begin{equation}\label{key}
		\begin{aligned}
			\left[f^L(x)\right]_{:,d/2, d/2}&= \left[ R^L \ast f^{L-1}(x) + \sigma \left(\widehat K^{d/2} \ast K^{d/2-1} \ast \cdots \ast K^{1}\ast x + B  + b^L \bm 1\right)\right]_{:,d/2, d/2} \\
			&= \left[ R^L \ast f^{L-1}(x) + \sigma \left(\widehat K^{d/2} \ast K^{d/2-1} \ast \cdots \ast K^{1}\ast x + \widehat \beta \bm 1\right) \right]_{:,d/2, d/2}\\
			&=  \left[ R^L \ast f^{L-1}(x)  \right]_{:,d/2, d/2} + \begin{pmatrix}
				\sigma \left( W{\mathcal V}(x) + \beta \right) \\
				\sigma(h\cdot {\mathcal V}(x) + c)\\
				\sigma( - h\cdot {\mathcal V}(x) - c)
			\end{pmatrix},
		\end{aligned}
	\end{equation}
	where $\widehat \beta = (\beta, c, -c)$. 
	Noticing that $\sigma(x) + \sigma(-x) = x$, finally we check that
	\begin{equation}\label{key}
		\begin{aligned}
			\widetilde f(x) &= a\cdot {\mathcal V}(f^L)(x) \\
			&= a \cdot {\mathcal V}\left( R^L \ast f^{L-1}(x)  \right) + (\alpha, -1, -1) \cdot  \begin{pmatrix}
				\sigma \left( W{\mathcal V}(x) + \beta \right) \\
				\sigma(h\cdot {\mathcal V}(x) + c)\\
				\sigma( - h\cdot {\mathcal V}(x) - c) 
			\end{pmatrix} \\
			&= l(x) + f_N( {\mathcal V}(x)) - \left( \sigma\left(l(x)\right) + \sigma\left(-l(x)\right) \right) \\
			& = f_N({\mathcal V}(x)), \quad \forall x \in \Omega,
		\end{aligned}
	\end{equation}
	where $l(x) = h \cdot {\mathcal V}(x) + c$.
	This completes the proof. 
\end{proof}

\paragraph{The approximation property of MgNet.} First, we introduce a typical MgNet~\cite{he2019mgnet,he2021interpretive} network used in image classification.
\begin{algorithm}[H]
	\footnotesize
	\caption{$u^J=\text{MgNet}(f)$}
	\label{alg:mgnet}
	\begin{algorithmic}[1]
		\STATE {\bf Input}: number of grids J, number of smoothing iterations $\nu_\ell$ for $\ell=1:J$, 
		number of channels $c_{f,\ell}$ for $f^\ell$ and $c_{u,\ell}$ for $u^{\ell,i}$ on $\ell$-th grid.
		\STATE Initialization:  $f^1 = f_{\rm in}(f)$, $u^{1,0}=0$
		\FOR{$\ell = 1:J$}
		\FOR{$i = 1:\nu_\ell$}
		\STATE Feature extraction (smoothing):
		\begin{equation}\label{eq:mgnet}
			u^{\ell,i} = u^{\ell,i-1} + \sigma \circ B^{\ell,i} \ast \sigma\left({f^\ell -  A^{\ell} \ast u^{\ell,i-1}}\right) \in \mathbb{R}^{c_{u,\ell}\times n_\ell\times m_\ell}.
		\end{equation}
		\ENDFOR
		\STATE Note: 
		$
		u^\ell= u^{\ell,\nu_\ell} 
		$
		\STATE Interpolation and restriction:
		\begin{equation}
			\label{eq:interpolation}
			u^{\ell+1,0} = \Pi_\ell^{\ell+1}\ast_2u^{\ell} \in \mathbb{R}^{c_{u,\ell+1}\times n_{\ell+1}\times m_{\ell+1}}.
		\end{equation}
		\begin{equation}
			\label{eq:restrict-f}
			f^{\ell+1} = R^{\ell+1}_\ell \ast_2 (f^\ell - A^\ell(u^{\ell})) + A^{\ell+1} \ast u^{\ell+1,0} \in \mathbb{R}^{c_{f,\ell+1}\times n_{\ell+1}\times m_{\ell+1}}.
		\end{equation}
		\ENDFOR
	\end{algorithmic}
\end{algorithm}
Here $ \Pi_\ell^{\ell+1}\ast_2$ and $R^{\ell+1}_\ell\ast_2$ in \eqref{eq:interpolation} and \eqref{eq:restrict-f}, respectively, which work as the coarsening operation, are defined as convolutions with stride two~\cite{goodfellow2016deep}.
As we do not include coarsening (sub-sampling or pooling ) layers in this study, we propose the following version of {the} feature extraction (smoothing) iteration in MgNet to study its approximation property.
\begin{description}
	\item[MgNet:]
	\begin{equation}\label{eq:mgnet_nopooling}
		\begin{cases}
			f^{\ell}(x) &= R^\ell \ast f^{\ell-1}(x) +  \sigma \left( B^\ell \ast\sigma \left(\theta^\ell \ast x- A^\ell \ast f^{\ell-1}(x) + a^\ell \bm 1 \right) + b^\ell \bm 1\right),  \\
			f(x) &= a \cdot {\mathcal V}(f^{L}(x)),
		\end{cases}
	\end{equation}
	for $\ell = 1:L$, where $f^0(x) = x \in \mathbb{R}^{d\times d}$, $A^\ell \in \mathbb{R}^{c_{\ell-1}\times C_{\ell} \times 3 \times 3}$, $B^\ell \in \mathbb{R}^{C_{\ell}\times c_{\ell} \times 3 \times 3}$, $R^\ell \in \mathbb{R}^{c_{\ell-1} \times c_{\ell} \times 1 \times 1}$, $a^\ell \in \mathbb{R}^{C_\ell}$, $\theta^{\ell} \in \mathbb{R}^{1\times C_\ell\times 3\times 3}$, $b^{c_\ell}$, and $a\in \mathbb{R}^{c_Ld^2}$.
\end{description}

As discussed in {relation to} ResNet and pre-act ResNet, we add extra $1\times 1$ convolutional kernels $R^\ell$ and $\theta^\ell$ in case of a change of channel. {Given that} we can recover pre-act ResNet by taking $\theta^\ell = 0$ in \eqref{eq:mgnet_nopooling}, we have the following theorem {pertaining to} the approximation property of MgNet.
\begin{theorem}
	Let $\Omega \subset \mathbb{R}^{d\times d}$ be bounded, and assume that $f: \Omega \mapsto \mathbb{R}$ with $\|f\|_{ \mathscr{K}_1(\mathbb D)} < \infty$. 
	Consider $\widetilde f(x)$ {as an} MgNet in \eqref{eq:mgnet_nopooling} with  any $R^\ell \in\mathbb{R}^{c_{\ell-1}\times c_{\ell} \times1\times1}$ given a prior, and hyperparameters that satisfy
	\begin{equation}
		\begin{aligned}
			&\text{depth (number of  blocks):} \quad &L& =\ \lfloor \lfloor d/2\rfloor/2\rfloor, \\
			&\text{width (number of channels):} \quad &c_\ell& = (4\ell + 1 )^2, ~C_\ell =2(4\ell-1)^2, 
		\end{aligned}
	\end{equation}
	for $\ell=1:L-1$, $C_L =2(4L-1)^2$, and $c_L = N+2$, such that 
	\begin{equation}\label{key}
		\|f - \widetilde f\|_{L^2(\Omega)}   \lesssim  N^{-\frac{1}{2}-\frac{3}{2d^2}} \|f\|_{ \mathscr{K}_1(\mathbb D)}.
	\end{equation}
\end{theorem}

\section{Conclusion}\label{sec:concl}
By carefully studying the decomposition theorem for convolutional kernels with large spatial size, 
we obtained the universal approximation property for a typical deep ReLU CNN structure. 
In general, we proved that deep { multi-channel ReLU} CNNs can represent one-hidden-layer ReLU NNs. Consequently, this representation result provides the same asymptotic approximation rate for deep ReLU CNNs as for one-hidden-layer ReLU NNs.
{Moreover}, we established approximation results for one version of ResNet, pre-act ResNet, and MgNet, based on the connections between these commonly used CNN models. 
This study provides new evidence of the theoretical foundation
of classical CNNs and popular architecture, such as ResNet, pre-act ResNet, and MgNet. 

Although the approximation properties do not show that CNNs should work better than DNNs in terms of approximation power, this study furthers the fields in understanding of CNNs in a significant way. For example, the success of CNNs may imply that the $\|f\|_{\mathscr{K}_1(\mathbb D)}$ norm is very small for the target feature extraction function $f$ in image classification or that $f$ belongs to a very special function class that can be efficiently represented or approximated by CNNs efficiently.
In addition, we anticipate that it will be possible to apply this kind of approximation results in designing new CNN structures, especially in the context of scientific computing~\cite{karniadakis2021physics}.
Furthermore, {as} the pooling operation plays a key role in practical CNNs, 
a natural future direction proceeding from this study is to derive approximation results for CNNs involving pooling layers.

\section*{Acknowledgements}
All the authors were partially supported by the Center for Computational Mathematics and Applications (CCMA) at The Pennsylvania State University.
The first author was also supported by the R.H. Bing Fellowship from The University of Texas at Austin. 
In addition, the third author was supported by the Verne M. William Professorship Fund from The Pennsylvania State University and by the National Science Foundation (Grant No. DMS-1819157 and DMS-2111387).

\section*{Conflict of interest}
The authors declare that they have no conflict of interest.

\section*{Data availability statement}
No datasets were generated or analyzed during the current study.

\newpage
\bibliographystyle{abbrv}
\bibliography{CNN_Approx.bib}

\begin{thebibliography}{10}
\providecommand{\url}[1]{{#1}}
\providecommand{\urlprefix}{URL }
\expandafter\ifx\csname urlstyle\endcsname\relax
  \providecommand{\doi}[1]{DOI~\discretionary{}{}{}#1}\else
  \providecommand{\doi}{DOI~\discretionary{}{}{}\begingroup
  \urlstyle{rm}\Url}\fi

\bibitem{arora2018understanding}
Arora, R., Basu, A., Mianjy, P., Mukherjee, A.: Understanding deep neural
  networks with rectified linear units.
\newblock In: International Conference on Learning Representations (2018)

\bibitem{bach2017breaking}
Bach, F.: Breaking the curse of dimensionality with convex neural networks.
\newblock The Journal of Machine Learning Research \textbf{18}(1), 629--681
  (2017)

\bibitem{bao2014approximation}
Bao, C., Li, Q., Shen, Z., Tai, C., Wu, L., Xiang, X.: Approximation analysis
  of convolutional neural networks.
\newblock work \textbf{65} (2014)

\bibitem{barron1993universal}
Barron, A.R.: Universal approximation bounds for superpositions of a sigmoidal
  function.
\newblock IEEE Transactions on Information Theory \textbf{39}(3), 930--945
  (1993)

\bibitem{cybenko1989approximation}
Cybenko, G.: Approximation by superpositions of a sigmoidal function.
\newblock Mathematics of Control, Signals and Systems \textbf{2}(4), 303--314
  (1989)

\bibitem{daubechies1992ten}
Daubechies, I.: Ten lectures on wavelets.
\newblock SIAM (1992)

\bibitem{e2021barron}
E, W., Ma, C., Wu, L.: The barron space and the flow-induced function spaces
  for neural network models.
\newblock Constructive Approximation pp. 1--38 (2021)

\bibitem{goodfellow2016deep}
Goodfellow, I., Bengio, Y., Courville, A.: Deep learning.
\newblock MIT Press (2016)

\bibitem{guhring2020error}
G{\"u}hring, I., Kutyniok, G., Petersen, P.: Error bounds for approximations
  with deep relu neural networks in $w^{s,p}$ norms.
\newblock Analysis and Applications \textbf{18}(05), 803--859 (2020)

\bibitem{guo2016convolutional}
Guo, X., Li, W., Iorio, F.: Convolutional neural networks for steady flow
  approximation.
\newblock In: Proceedings of the 22nd ACM SIGKDD international conference on
  knowledge discovery and data mining, pp. 481--490 (2016)

\bibitem{he2021relu}
He, J., Li, L., Xu, J.: Relu deep neural networks from the hierarchical basis
  perspective.
\newblock arXiv preprint arXiv:2105.04156  (2021)

\bibitem{he2020relu}
He, J., Li, L., Xu, J., Zheng, C.: Relu deep neural networks and linear finite
  elements.
\newblock Journal of Computational Mathematics \textbf{38}(3), 502--527 (2020)

\bibitem{he2019mgnet}
He, J., Xu, J.: Mgnet: A unified framework of multigrid and convolutional
  neural network.
\newblock Science China Mathematics pp. 1--24 (2019)

\bibitem{he2021interpretive}
He, J., Xu, J., Zhang, L., Zhu, J.: An interpretive constrained linear model
  for resnet and mgnet.
\newblock arXiv preprint arXiv:2112.07441  (2021)

\bibitem{he2016deep}
He, K., Zhang, X., Ren, S., Sun, J.: Deep residual learning for image
  recognition.
\newblock In: Proceedings of the IEEE conference on computer vision and pattern
  recognition, pp. 770--778 (2016)

\bibitem{he2016identity}
He, K., Zhang, X., Ren, S., Sun, J.: Identity mappings in deep residual
  networks.
\newblock In: European conference on computer vision, pp. 630--645. Springer
  (2016)

\bibitem{hornik1989multilayer}
Hornik, K., Stinchcombe, M., White, H.: Multilayer feedforward networks are
  universal approximators.
\newblock Neural Networks \textbf{2}(5), 359--366 (1989)

\bibitem{huang2017densely}
Huang, G., Liu, Z., Van Der~Maaten, L., Weinberger, K.Q.: Densely connected
  convolutional networks.
\newblock In: Proceedings of the IEEE conference on computer vision and pattern
  recognition, pp. 4700--4708 (2017)

\bibitem{karniadakis2021physics}
Karniadakis, G.E., Kevrekidis, I.G., Lu, L., Perdikaris, P., Wang, S., Yang,
  L.: Physics-informed machine learning.
\newblock Nature Reviews Physics \textbf{3}(6), 422--440 (2021)

\bibitem{klusowski2018approximation}
Klusowski, J.M., Barron, A.R.: Approximation by combinations of relu and
  squared relu ridge functions with $\ell^{1}$ and $\ell^{0}$ controls.
\newblock IEEE Transactions on Information Theory \textbf{64}(12), 7649--7656
  (2018)

\bibitem{kohler2020statistical}
Kohler, M., Langer, S.: Statistical theory for image classification using deep
  convolutional neural networks with cross-entropy loss.
\newblock arXiv preprint arXiv:2011.13602  (2020)

\bibitem{krizhevsky2012imagenet}
Krizhevsky, A., Sutskever, I., Hinton, G.E.: Imagenet classification with deep
  convolutional neural networks.
\newblock Advances in Neural Information Processing Systems \textbf{25},
  1097--1105 (2012)

\bibitem{kumagai2020universal}
Kumagai, W., Sannai, A.: Universal approximation theorem for equivariant maps
  by group cnns.
\newblock arXiv preprint arXiv:2012.13882  (2020)

\bibitem{lecun2015deep}
LeCun, Y., Bengio, Y., Hinton, G.: Deep learning.
\newblock Nature \textbf{521}(7553), 436--444 (2015)

\bibitem{lecun1998gradient}
LeCun, Y., Bottou, L., Bengio, Y., Haffner, P.: Gradient-based learning applied
  to document recognition.
\newblock Proceedings of the IEEE \textbf{86}(11), 2278--2324 (1998)

\bibitem{leshno1993multilayer}
Leshno, M., Lin, V.Y., Pinkus, A., Schocken, S.: Multilayer feedforward
  networks with a nonpolynomial activation function can approximate any
  function.
\newblock Neural Networks \textbf{6}(6), 861--867 (1993)

\bibitem{lin2021universal}
Lin, S.B., Wang, K., Wang, Y., Zhou, D.X.: Universal consistency of deep
  convolutional neural networks.
\newblock arXiv preprint arXiv:2106.12498  (2021)

\bibitem{lu2017expressive}
Lu, Z., Pu, H., Wang, F., Hu, Z., Wang, L.: The expressive power of neural
  networks: A view from the width.
\newblock In: Advances in Neural Information Processing Systems, pp. 6231--6239
  (2017)

\bibitem{montufar2014number}
Montufar, G.F., Pascanu, R., Cho, K., Bengio, Y.: On the number of linear
  regions of deep neural networks.
\newblock In: Advances in Neural Information Processing Systems, pp. 2924--2932
  (2014)

\bibitem{nair2010rectified}
Nair, V., Hinton, G.E.: Rectified linear units improve restricted boltzmann
  machines.
\newblock In: Proceedings of the 27th International Conference on International
  Conference on Machine Learning, pp. 807--814 (2010)

\bibitem{oono2019approximation}
Oono, K., Suzuki, T.: Approximation and non-parametric estimation of
  resnet-type convolutional neural networks.
\newblock In: International Conference on Machine Learning, pp. 4922--4931.
  PMLR (2019)

\bibitem{opschoor2020deep}
Opschoor, J.A., Petersen, P.C., Schwab, C.: Deep relu networks and high-order
  finite element methods.
\newblock Analysis and Applications pp. 1--56 (2020)

\bibitem{paszke2019pytorch}
Paszke, A., Gross, S., Massa, F., Lerer, A., Bradbury, J., Chanan, G., Killeen,
  T., Lin, Z., Gimelshein, N., Antiga, L., et~al.: Pytorch: An imperative
  style, high-performance deep learning library.
\newblock Advances in Neural Information Processing Systems \textbf{32},
  8026--8037 (2019)

\bibitem{petersen2020equivalence}
Petersen, P., Voigtlaender, F.: Equivalence of approximation by convolutional
  neural networks and fully-connected networks.
\newblock Proceedings of the American Mathematical Society \textbf{148}(4),
  1567--1581 (2020)

\bibitem{poggio2017and}
Poggio, T., Mhaskar, H., Rosasco, L., Miranda, B., Liao, Q.: Why and when can
  deep-but not shallow-networks avoid the curse of dimensionality: a review.
\newblock International Journal of Automation and Computing \textbf{14}(5),
  503--519 (2017)

\bibitem{shen2019nonlinear}
Shen, Z., Yang, H., Zhang, S.: Nonlinear approximation via compositions.
\newblock Neural Networks \textbf{119}, 74--84 (2019)

\bibitem{siegel2020approximation}
Siegel, J.W., Xu, J.: Approximation rates for neural networks with general
  activation functions.
\newblock Neural Networks \textbf{128}, 313--321 (2020)

\bibitem{siegel2021characterization}
Siegel, J.W., Xu, J.: Characterization of the variation spaces corresponding to
  shallow neural networks.
\newblock arXiv preprint arXiv:2106.15002  (2021)

\bibitem{siegel2021improved}
Siegel, J.W., Xu, J.: Improved approximation properties of dictionaries and
  applications to neural networks.
\newblock arXiv preprint arXiv: 2101.12365  (2021)

\bibitem{siegel2022high}
Siegel, J.W., Xu, J.: High-order approximation rates for shallow neural
  networks with cosine and reluk activation functions.
\newblock Applied and Computational Harmonic Analysis \textbf{58}, 1--26 (2022)

\bibitem{tan2019efficientnet}
Tan, M., Le, Q.: Efficientnet: Rethinking model scaling for convolutional
  neural networks.
\newblock In: International Conference on Machine Learning, pp. 6105--6114.
  PMLR (2019)

\bibitem{telgarsky2016benefits}
Telgarsky, M.: Benefits of depth in neural networks.
\newblock Journal of Machine Learning Research \textbf{49}(June), 1517--1539
  (2016)

\bibitem{xu2020finite}
Xu, J.: Finite neuron method and convergence analysis.
\newblock Communications in Computational Physics \textbf{28}(5), 1707--1745
  (2020)

\bibitem{yarotsky2017error}
Yarotsky, D.: Error bounds for approximations with deep relu networks.
\newblock Neural Networks \textbf{94}, 103--114 (2017)

\bibitem{zhou2018deep}
Zhou, D.X.: Deep distributed convolutional neural networks: Universality.
\newblock Analysis and Applications \textbf{16}(06), 895--919 (2018)

\bibitem{zhou2020universality}
Zhou, D.X.: Universality of deep convolutional neural networks.
\newblock Applied and Computational Harmonic Analysis \textbf{48}(2), 787--794
  (2020)

\end{thebibliography}

\end{document}